\documentclass[10pt,journal,compsoc]{IEEEtran}
\ifCLASSOPTIONcompsoc
\usepackage[nocompress]{cite}
\else
\usepackage{cite}
\fi

\ifCLASSINFOpdf
\else
\fi
\usepackage{amsmath,amsfonts}
\usepackage{algorithmic}
\usepackage{algorithm}
\usepackage{array}
\usepackage{textcomp}
\usepackage{stfloats}
\usepackage{url}
\usepackage{verbatim}
\usepackage{graphicx}
\usepackage{multirow}
\usepackage{booktabs}
\usepackage{subcaption}
\usepackage{enumitem}
\usepackage{color}
\usepackage{ragged2e} 
\usepackage{url}      

\usepackage{bm}
\usepackage{amsmath}
\usepackage{amssymb}
\usepackage{amsthm} 
\usepackage[most]{tcolorbox}

\newtheorem{proposition}{Proposition}[section]
\newtheorem{assumption}{Assumption}[section]

\newtheorem{definition}{Definition}[section]

\usepackage[table]{xcolor}

\definecolor{growthgreen}{RGB}{0, 160, 0}
\definecolor{SoftCyan}{RGB}{238,246,248}
\definecolor{SoftPeach}{RGB}{248,239,231}
\definecolor{SoftGreen}{RGB}{238,246,239}

\definecolor{iceblue}{RGB}{235, 246, 255}

\theoremstyle{remark} 

\def\BibTeX{{\rm B\kern-.05em{\sc i\kern-.025em b}\kern-.08em
		T\kern-.1667em\lower.7ex\hbox{E}\kern-.125emX}}

\definecolor{rblue}{rgb}{0,0.5,1}
\definecolor{awesome}{rgb}{1.0, 0.13, 0.32}
\definecolor{hollywoodcerise}{rgb}{0.96, 0.0, 0.63}
\definecolor{lasallegreen}{rgb}{0.03, 0.47, 0.19}
\definecolor{hanpurple}{rgb}{0.32, 0.09, 0.98}
\definecolor{green(pigment)}{rgb}{0.0, 0.65, 0.31}
\usepackage[pagebackref=false,breaklinks=true,colorlinks,bookmarks=false]{hyperref}
\hypersetup{colorlinks=true,linkcolor={red},citecolor={hanpurple},urlcolor={magenta}}

\renewcommand{\footnoterule}{%
  \kern 4pt                         
  \hrule width 0.4\textwidth height 0.4pt 
  \kern 4pt                        
}

\begin{document}

\title{When Detectors Forget Forensics: Blocking Semantic Shortcuts for Generalizable AI-Generated Image Detection}

\author{
	Chao~Shuai,
    Shaojing Fan,
    Chenlin Zou,
    Bin Gong,
    Weichen Lian,
    Xiuli Bi,
    Zhenguang~Liu{$^\dagger$}, \textit{Member, IEEE},
    Zhongjie~Ba, \textit{Member, IEEE},
    Kui~Ren, \textit{Fellow, IEEE}
	\IEEEcompsocitemizethanks{\IEEEcompsocthanksitem Chao Shuai, Chenlin Zou, Bin Gong, Weichen Lian, Zhongjie Ba, Zhenguang Liu, and Kui Ren are with the State Key Laboratory of Blockchain and Data Security, Zhejiang University, Hangzhou, China. Email: \href{chaoshuai@zju.edu.cn}{chaoshuai@zju.edu.cn}, 
    \href{chenlinzou@zju.edu.cn}{chenlinzou@zju.edu.cn},
    \href{gong_bin@zju.edu.cn}{gong\_bin@zju.edu.cn},
    \href{kevinbaylor@163.com}{kevinbaylor@163.com},
    \href{zhongjieba@zju.edu.cn}{zhongjieba@zju.edu.cn}, \href{liuzhenguang2008@gmail.com}{liuzhenguang2008@gmail.com}, \href{kuiren@zju.edu.cn}{kuiren@zju.edu.cn}. Shaojing Fan is with the National University of Singapore, Singapore. 
    Xiuli Bi is with the Chongqing University of Posts and Telecommunications, Chongqing, China. E-mail: \href{dcsfs@nus.edu.sg}{dcsfs@nus.edu.sg}, \href{bixl@cqupt.edu.cn}{bixl@cqupt.edu.cn}.
    \IEEEcompsocthanksitem {$^\dagger$}Corresponding author: Zhenguang Liu.}
}

\IEEEtitleabstractindextext{
\begin{abstract}
\justifying 
The growing realism of generative models has blurred the boundary between real and synthetic content, posing significant challenges to reliable AI-generated image detection. Although large-scale pre-trained Vision Foundation Models have advanced detection capability, their generalization to images from unseen generation pipelines remains inadequate. In this paper, we identify, for the first time, a key failure mechanism, termed \emph{semantic fallback}, wherein forensic fine-tuning fails to fully reshape the representation space. Consequently, the resulting representations remain organized along high-level semantic structures rather than manipulation-specific forensic cues. Building on this insight, we propose a \textbf{Geometric Semantic Decoupling (GSD)} framework, which explicitly suppresses semantically dominant directions, thereby promoting invariant forensic representations. Specifically, GSD leverages a frozen CLIP encoder to estimate the dominant semantic subspace via Singular Value Decomposition (SVD). It then suppresses the semantic components through a geometry-constrained formulation with the suppression strength adaptively modulated across samples and layers. We further introduce a mini-batch SVD approximation strategy that amortizes subspace estimation, achieving over a 15× reduction in computational overhead while preserving effectiveness. Finally, considering practical scenarios spanning both large-scale and online evaluation, we develop three inference protocols, batch, per-sample, and reference-based inference, and demonstrate that they induce consistent semantic decoupling, yielding a stable forgery-oriented feature manifold. Empirical results show that our method consistently achieves state-of-the-art performance across four benchmarks (38 datasets), with improvements from 93.4 to 96.3 in video-level AUC for face forgery detection and from 93.6 to 96.3 in accuracy for synthetic image detection. 



\end{abstract}

}

\maketitle
\IEEEdisplaynontitleabstractindextext
\IEEEpeerreviewmaketitle

\IEEEraisesectionheading{\section{Introduction}\label{sec:introduction}}
\IEEEPARstart{T}{he} recent rise of advanced generative AI~\cite{midjourney, sd} has raised serious questions about the authenticity of digital content, making it increasingly difficult to tell what is real and what is artificially generated. This technological advancement has moved beyond a security concern and now results in serious real-world consequences. Recent incidents involving deepfake-based corporate scams and privacy violations underscore the fragility of the current digital ecosystem~\cite{arup2024, swift2024}. The wide spread of AI-generated content creates a general sense of skepticism and threatens the credibility of all media. Addressing this challenge requires the development of a universal and adaptive detection mechanism that can remain effective as AI generation techniques continue to evolve rapidly.

\begin{figure*}[!t]
  \centering
    \includegraphics[width=\linewidth]{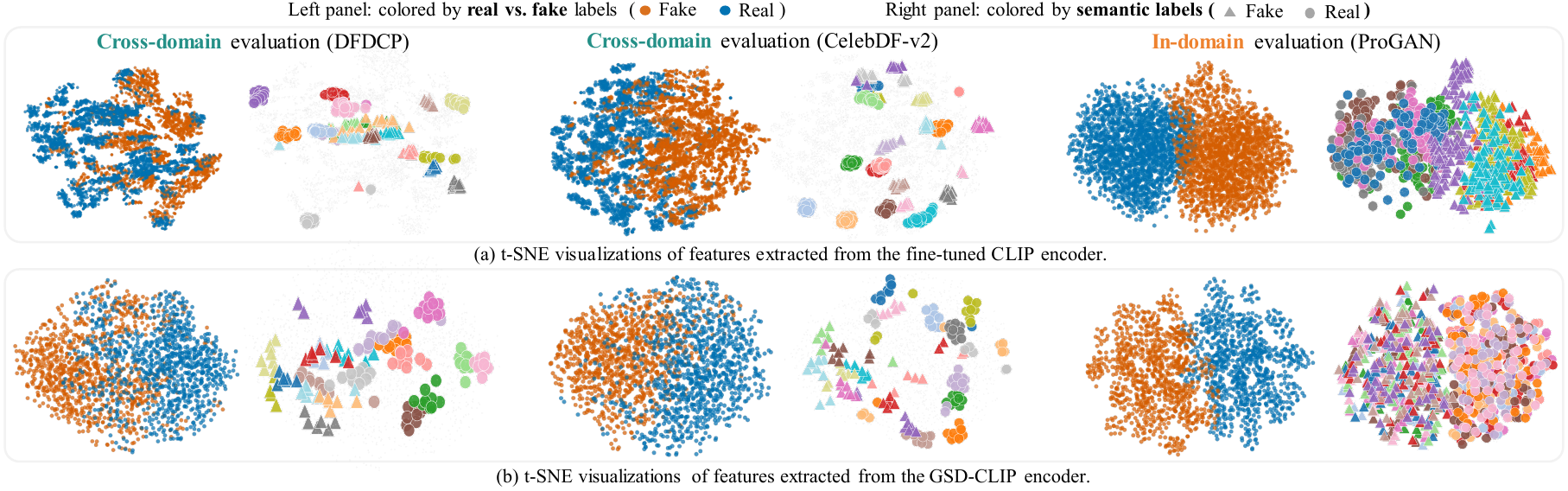}
    \caption{t-SNE~\cite{tsne} visualizations of feature representations extracted from the fine-tuned CLIP encoder and GSD-CLIP. 
The top row corresponds to the fine-tuned CLIP encoder, and the bottom row corresponds to GSD-CLIP. 
The three columns show results on the in-domain ProGAN~\cite{progan} dataset and two cross-domain benchmarks, CelebDF-v2~\cite{celebdf} and DFDCP~\cite{dfdcp}. 
For each dataset, the left panel is colored by forgery labels (real vs.\ fake), while the right panel is colored by semantic labels. Compared with the fine-tuned CLIP encoder, which retains compact semantic clusters across both in-domain and cross-domain settings, GSD-CLIP substantially weakens identity- or content-driven semantic grouping and yields a representation space more aligned with fake/real discrimination.}
    \label{sec1-fig1}
\end{figure*}

In response, the digital forensic community has undergone a significant methodological evolution, progressing from early statistical artifact analysis toward deep learning paradigms. A consensus has emerged among state-of-the-art detectors~\cite{effort,x2dfd,fakeradar,fcg,deepshield,vbsat}, which mainly leverage large-scale pre-trained Vision Foundation Models (VFMs), such as CLIP~\cite{clip}. These models provide robust feature representations that have significantly boosted detection accuracy on standard benchmarks. However, a persistent paradox remains: while these detectors achieve notable performance on seen training distributions, they suffer from catastrophic performance degradation when exposed to unseen generation techniques~\cite{deepfake, timely, x2dfd}. In this work, we posit that this failure arises from a fundamental conflict between the semantic-centric objective of pre-training and the artifact-centric objective of forensics. Supporting this view, prior observations~\cite{effort} show that semantic priors learned by foundation models concentrate in dominant representational directions, while transferable forensic artifacts occupy weaker, more domain-sensitive subspaces.

Based on the above, \textit{we hypothesize that when forensic artifacts are subtle or poorly transferable, the model falls back on the strong semantic priors (e.g., identity or object category) embedded in the foundation model, thereby overshadowing faint manipulation signals.} To test this hypothesis, we examine the generalization bottleneck in vision transformer--based deepfake detectors by comparing feature distributions from a frozen pre-trained CLIP encoder and its fully fine-tuned counterpart using t-SNE~\cite{tsne} visualization. Specifically, we conduct this analysis on ten datasets covering both face forgery and general synthetic-image detection scenarios, where semantic labels such as face identity or object category are available. For face forgery detection, we follow the standard training protocol on FaceForensics++~\cite{faceforensics++} and evaluate both in-domain and cross-domain datasets. For general synthetic-image detection, we follow the UniversalFakeDetect protocol and train on the ProGAN subset, then evaluate on both seen and unseen generator domains. Figure~\ref{sec1-fig1}(a) presents representative results, including the cross-domain face forgery benchmarks CelebDF-v2 and DFDCP, and the in-domain ProGAN results for synthetic-image detection. More results are provided in Appendix~\ref{secA-3}. From this analysis, we make the observations that the fine-tuned detector does not fully reshape its representation space around manipulation-specific evidence. Instead, a clear semantic organization remains visible across both in-domain and cross-domain settings: samples sharing the same identity or semantic category still form compact local groups in the embedding space even after fine-tuning. We term this phenomenon \textit{semantic fallback}, where the semantic structure inherited from CLIP pre-training is not removed by downstream forensic training, but persists as a dominant geometric prior in the learned representation.


\textit{How does semantic fallback affect the generalization of AI-generated image detection?} We address this question through both empirical and theoretical analyses in Section~\ref{sec3}. Empirically, we show that: (a) semantic fallback persists after forensic fine-tuning and becomes more pronounced on unseen domains; (b) detection failures are more strongly associated with samples that retain higher semantic occupancy in the inherited CLIP subspace; and (c) suppressing semantic components improves cross-domain performance in a linear-probe setting. Theoretically, we analyze the target-risk trade-off introduced by semantic suppression and derive the existence of a non-trivial optimal suppression strength that balances semantic interference removal with the preservation of discriminative forensic cues. These findings demonstrate that semantic fallback is a fundamental generalization bottleneck for VFM-based deepfake detectors, rather than a merely incidental property of the learned representation.


Motivated by these observations, we propose a \textbf{Geometric Semantic Decoupling (GSD)} framework, which constrains the detector to learn in the \textbf{semantic null space}. Specifically, we introduce a frozen CLIP encoder as a semantic anchor, estimate semantically dominant directions via Singular Value Decomposition (SVD)~\cite{svd}, and explicitly subtract the semantic component from the learned forensic representation. Our analysis in Section~3 shows that semantic fallback becomes more pronounced in deeper CLIP layers, while prior studies have suggested that shallow representations are more important for preserving low-level forgery artifacts. Motivated by this, we insert GSD into the last four transformer layers and design an adaptive semantic suppression strategy that adjusts the layer-wise suppression strength according to the semantic occupancy of the current forensic feature.
Furthermore, because estimating a semantic subspace via per-sample SVD is computationally expensive, we introduce a mini-batch-based approximation, termed \textit{Batch-SVD}, which replaces repeated sample-wise decompositions with a shared subspace estimate and reduces the computational cost by more than $15$ times. Finally, to support both batch evaluation and realistic streaming scenarios, we develop three inference protocols, namely batch inference, per-sample inference, and reference-based inference. Benefiting from the fact that GSD establishes a stable forgery-oriented manifold in the decoupled feature space during training, these protocols remain effective even though they estimate the semantic subspace differently, because they all operate under the same semantic geometry induced by the frozen CLIP encoder. 

Interestingly, after suppressing high-level semantic components, a clear structural change can be observed in the feature space. Fake images begin to express forgery-related cues more prominently, which gradually become the dominant factors shaping the representation, as evidenced by the increased dispersion of fake samples in the t-SNE visualization (see Figure~\ref{sec1-fig1}(b)). In contrast, real images, which lack forgery-specific patterns, remain relatively identity-clustered even after semantic suppression. In summary, our main contributions are:
\begin{itemize}
    \item To the best of our knowledge, we are the first to identify \textit{semantic fallback} as a key cause of generalization failure, where VFM-based detectors regress to pre-trained semantic priors (\textit{e.g.,} identity cues) in unseen domains. We further show that this failure is systematic and more pronounced on unseen AI-generated image detection, revealing a previously underexplored limitation of existing VFM-based detectors.

    \item To overcome semantic fallback, we introduce \textbf{Geometric Semantic Decoupling (GSD)}, which uses a frozen CLIP encoder to estimate the semantic subspace and geometrically suppress dominant semantic components in learned representations. In this way, GSD encourages the detector to focus on manipulation-specific forensic signals rather than high-level semantic cues. Moreover, we introduce an adaptive subtraction coefficient $\lambda$, which allows GSD to perform layer-aware, sample-adaptive semantic suppression according to the semantic occupancy of the current representations.

    \item We exploit the observation that the semantic subspace in CLIP is highly structured and stable, and introduce a mini-batch SVD approximation, termed \textit{Batch-SVD}, to replace computationally expensive per-sample subspace estimation, improving both training and inference efficiency by more than $15$ times. We further develop three complementary inference protocols, namely batch inference, per-sample inference, and reference-based inference, enabling GSD to support both large-scale batch evaluation and realistic streaming detection.
\end{itemize}
\section{Related work}
\label{sec2}
\label{rlk}
\subsection{Generalizable AI-generated Image Detection} 
Early research on AI-generated image detection primarily exploited manipulation-specific artifacts. For synthetic image detection, prior works often focus on generator-induced traces, such as frequency-spectrum irregularities~\cite{cnnspot,f3net,freqdebias,frequency}, local texture patterns~\cite{patchfor,spatial,two-branch}, and gradient-domain inconsistencies~\cite{lgrad}. For deepfake detection, methods commonly leverage face-specific manipulation cues, including boundary blending artifacts~\cite{face-x-ray,chen,lgrad} and temporal or identity inconsistencies~\cite{detection,multi,lips,ict,id-reveal,dual}. However, these supervised detectors often suffer from limited cross-domain generalization, as they tend to overfit dataset- or manipulation-specific patterns. To alleviate this issue, data synthesis strategies have been explored, such as generating blending-based ``pseudo-fakes'' to encourage learning more generic forensic boundaries  rather than artifacts tied to specific generation pipelines~\cite{sbi,prodet,fsgd}.

With the rise of large-scale pre-training, Vision Foundation Models (VFMs), \textit{e.g.,} CLIP~\cite{clip}, have become a dominant backbone for AI-generated image detection due to their rich open-world representations. Early studies demonstrate that frozen VFM features already provide strong baselines, particularly in synthetic image detection~\cite{univfd}. Building upon this, a series of works adapt these representations via parameter-efficient strategies. For synthetic image detection, methods such as FatFormer~\cite{fatformer}, C2P-CLIP~\cite{c2p}, and CLIPMoLE~\cite{clipmode} introduce adapters, prompt tuning, or Mixture-of-Experts mechanisms to better align features with detection objectives. For deepfake detection, approaches like Forensics Adapter~\cite{foradapter} and EFFORT~\cite{effort} explore adapter-based tuning and orthogonality constraints to preserve useful pre-trained knowledge while enhancing forgery-relevant cues. In video settings, deepfake detection further incorporates temporal modeling modules to capture cross-frame inconsistencies~\cite{fcg,vbsat,deepshield}. More recently, MLLM-based frameworks extend detection toward interpretable, reasoning-oriented paradigms~\cite{x2dfd,fakeshield}.

Despite improved robustness, VFM-based methods introduce a fundamental challenge. Vision-language pre-training is inherently driven by semantic alignment, encouraging representations to preserve high-level content such as object category, identity, and scene semantics. However, forgery detection relies on subtle manipulation traces that are weakly correlated with such semantics. As a result, under domain shift, detector representations can remain dominated by the semantic geometry inherited from pre-training, causing forensic evidence to be overshadowed by semantically salient but task-irrelevant factors. We term this phenomenon the \emph{semantic fallback} effect. To address it, we propose a geometry-driven semantic decoupling mechanism that explicitly suppresses semantically dominant components in VFM representations, thereby promoting more transferable forgery-oriented features.

\subsection{Disentanglement and Semantic Debiasing} 
Some studies have attempted to disentangle forensic cues from confounding semantic factors, such as identity, background content, and other high-level visual attributes. The common intuition behind these methods is that improving generalization requires detectors to reduce their reliance on semantically salient yet task-irrelevant information, and instead emphasize manipulation-relevant evidence. Some introduce explicit constraints or disentanglement objectives to encourage forgery-relevant and manipulation-invariant representations~\cite{cfm,udd, ucf,dffake}. For example, UCF~\cite{ucf} employs a dedicated content reconstruction branch to absorb high-level content features, thereby leaving the main branch more focused on forensic evidence. DiffusionFake~\cite{dffake} leverages frozen Stable Diffusion inversion to capture source- and target-related features for reconstruction, enabling the detector to learn richer and more disentangled representations that are more resilient to unseen forgeries. UDD~\cite{udd} weakens semantic dependencies through transformer token shuffling, forcing the network to rely less on stable semantic structure and more on manipulation traces.  Besides, ExposDe~\cite{exposde} adopts mutual-information-based regularization to encourage the encoding of complementary features from different dimensions.

Despite their methodological diversity, these approaches mitigate semantic bias only in an indirect manner, \emph{e.g.}, through auxiliary branches, reconstruction  priors, or soft regularization objectives. As a result, semantic interference is often addressed only indirectly. In contrast, our work departs from this line of research by explicitly examining how semantic priors are entangled with forensic evidence in the learned representation, and in doing so identifies \emph{semantic fallback} as a fundamental failure mechanism of VFM-based detection. Building on this insight, we estimate semantically dominant directions from a frozen CLIP encoder and remove them through geometric projection, thereby enforcing semantic suppression directly in feature space.

Among existing methods, EFFORT~\cite{effort} is particularly related to our work in that it also employs SVD~\cite{svd} to structure the feature space. However, its motivation is fundamentally different from ours. EFFORT assumes that the dominant principal components inherited from large-scale pre-training encode beneficial semantic priors and therefore should be preserved. It thus freezes the principal subspace and adapts only the residual components, using the orthogonal complement to learn forgery-related patterns while retaining the original semantic knowledge of the pre-trained model. In contrast, our empirical analysis reveals the opposite phenomenon: under domain shift, the semantically dominant directions inherited from vision-language pre-training can hinder transferable forensic discrimination by biasing the detector toward semantic fallback. Accordingly, rather than preserving the dominant semantic subspace, our method explicitly suppresses it.
\section{Semantic Fallback as a Bottleneck for AI-generated Image Detection}
\label{sec3}

\begin{table}[t]
\centering
\caption{Distribution-level semantic similarity gap $\Delta S$ measured on CLS features from different layers of the frozen CLIP encoder ($CLIP_{fz}$) and the fine-tuned CLIP encoder ($CLIP_{ft}$). A larger $\Delta S$ indicates a stronger identity-coupled semantic structure in the learned representation.} 
\label{sec3-tab1}
\renewcommand{\arraystretch}{0.85}
\resizebox{\linewidth}{!}{%
\begin{tabular}{lccccc}
\toprule
\textbf{Layer \textit{L}} & \textbf{2} & \textbf{4} & \textbf{8} & \textbf{12} & \textbf{24} \\
\midrule

\multicolumn{6}{c}{\textbf{FaceForensics++}} \\
\midrule
$CLIP_{fz}$ & 0.153 & 0.102 & 0.153 & 0.145 & 0.333 \\
$CLIP_{ft}$ & 0.152 & 0.110 & 0.142 & 0.136 & 0.317 \\
\midrule

\multicolumn{6}{c}{\textbf{Celeb-DF v2}} \\
\midrule
$CLIP_{fz}$ & 0.113 & 0.078 & 0.130 & 0.123 & 0.385 \\
$CLIP_{ft}$ & 0.111 & 0.085 & 0.127 & 0.115 & 0.398 \\
\midrule

\multicolumn{6}{c}{\textbf{DFDC}} \\
\midrule
$CLIP_{fz}$ & 0.176 & 0.156 & 0.204 & 0.164 & 0.392 \\
$CLIP_{ft}$ & 0.173 & 0.167 & 0.199 & 0.156 & 0.413 \\
\bottomrule
\end{tabular}%
}
\end{table}

Recent studies~\cite{ucf, udd, effort} have suggested that AI-generated image detectors are influenced not only by manipulation traces, but also by high-level semantic features. However, the mechanism underlying this phenomenon remains poorly understood: \emph{how do semantic priors enter the detection process, and why do they undermine forensic generalization under domain shift?} In this work, we argue that the generalization bottleneck lies in a residual semantic bias inherited from foundation-model pre-training. Specifically, even after downstream forensic fine-tuning, the learned representation may remain substantially organized by high-level semantic structure, such as identity, object category, or scene content, rather than being fully reconfigured around semantic-invariant forensic evidence. Consequently, when manipulation traces become weak, unstable, or less transferable under domain shift, the detector may regress to these semantically dominant directions and base its decisions on content priors rather than on intrinsic forgery cues. We refer to this failure mechanism as \emph{semantic fallback}.

A central difficulty, however, is that the mere existence of semantic structure in the feature space does not by itself imply a harmful effect on generalization. Semantic organization may simply coexist with forensic discrimination, without actively interfering with it. Therefore, before introducing our method, the issue must be resolved at two levels: \emph{(1) whether semantic fallback persists as a stable and quantifiable property of fine-tuned representations, rather than being an artifact of qualitative visualization; and (2) whether, once present, it is merely correlated with cross-domain failure or instead actively contributes to the degradation of transferable detection performance.} To answer these questions, we conduct a three-step analysis in this section. We first show that semantic fallback remains a persistent property of fine-tuned VFM representations and becomes more pronounced under cross-domain evaluation. We then move from correlation to intervention, demonstrating that samples with stronger semantic dominance are more prone to prediction failure, and that directly weakening the semantic component improves cross-domain discrimination even in a simple linear-probe setting. Finally, we introduce a coupled representation view together with a margin decomposition analysis, which provides a theoretical account of why semantic suppression is beneficial for generalization.

\subsection{Semantic Fallback in Fine-Tuned Representations}
We begin by asking whether forensic fine-tuning truly removes the semantic organization inherited from foundation-model pre-training. Figure~\ref{sec1-fig1} provides an initial indication that this is not the case: although the fine-tuned CLIP encoder exhibits visible real/fake separation, its feature space still retains clear content-aware local structure, with samples from the same identity remaining clustered. However, because such visualizations are inherently qualitative, a quantitative characterization is necessary to establish whether semantic fallback is a genuine property of the learned representation, and quantify the strength of identity-aware structure in the feature space.

\begin{definition}[Identity Similarity Gap]
Given a set of CLS features $\{\bm{f}_i\}_{i=1}^{N}$ with identity labels $\{y_i\}_{i=1}^{N}$, where $\bm{f}_i \in \mathbb{R}^{d}$ and $y_i \in \{1, \dots, C\}$, the identity similarity gap is defined as $\Delta S \triangleq S_{\mathrm{intra}} - S_{\mathrm{inter}}$,
where $S_{\mathrm{intra}} \triangleq \mathbb{E}\left[\cos(\bm{f}_i, \bm{f}_j)\mid y_i = y_j\right]$ and $S_{\mathrm{inter}} \triangleq \mathbb{E}\left[\cos(\bm{f}_i, \bm{f}_j)\mid y_i \neq y_j\right].$ A larger $\Delta S$ indicates that samples sharing the same identity are consistently more similar to each other than to samples from different identities, and therefore implies a stronger identity-coupled semantic structure in the representation space
\end{definition}


Using the identity similarity gap $\Delta S$ defined above, we evaluate both the frozen CLIP encoder ($\mathrm{CLIP}_{fz}$) and its fine-tuned counterpart ($\mathrm{CLIP}_{ft}$). Specifically, we extract the CLS token from layers $L\in\{2,4,8,12,24\}$ (layer 24 denotes the final layer) and compute $\Delta S$ on the in-domain FaceForensics++ dataset as well as two cross-domain benchmarks Celeb-DF v2 and DFDC (results in Table~\ref{sec3-tab1}). We find that: 1) the feature distributions extracted from every layer of both $\mathrm{CLIP}_{fz}$ and $\mathrm{CLIP}_{ft}$ exhibit highly consistent semantic structure across all datasets. This indicates downstream forensic fine-tuning does not remove the semantic structure inherited from CLIP pre-training; rather, such structure remains broadly preserved across layers in both in-domain and cross-domain settings. 2) the behavior at the final layer differs between in-domain and cross-domain evaluation. On the in-domain FaceForensics++ dataset, the final-layer $\Delta S$ decreases after fine-tuning (from 0.333 to 0.317), suggesting that the model partially weakens semantic priors  when familiar manipulation cues are available. By contrast, on the unseen Celeb-DF v2 and DFDC benchmarks, the final-layer $\Delta S$ is maintained or further increased after fine-tuning (from 0.385 to 0.398, and from 0.392 to 0.413, respectively). This suggests that cross-domain failure is not merely accompanied by weaker forgery cues, but by the continued dominance of pre-trained semantic structure.

This phenomenon runs counter to the original goal of general AI-generated image detection. A desirable detector should organize its representation space around manipulation-related evidence that is transferable across semantic content, rather than around identity- or content-specific semantics. In other words, if two images are synthesized by the same forgery algorithm but differ in subject identity, a robust detector should still extract highly similar forensic features from them, because the shared manipulation traces are the factors truly relevant to detection. From this perspective, the feature space of an ideal detector should exhibit weak dependence on high-level semantics, which would be reflected by a relatively low $\Delta S$. Therefore, the consistently non-trivial and even amplified $\Delta S$ observed in cross-domain settings indicates a clear mismatch between the current representation geometry and the desired behavior of a generalizable forensic model.

\begin{figure}[!t]
  \centering
    \includegraphics[width=0.95\linewidth]{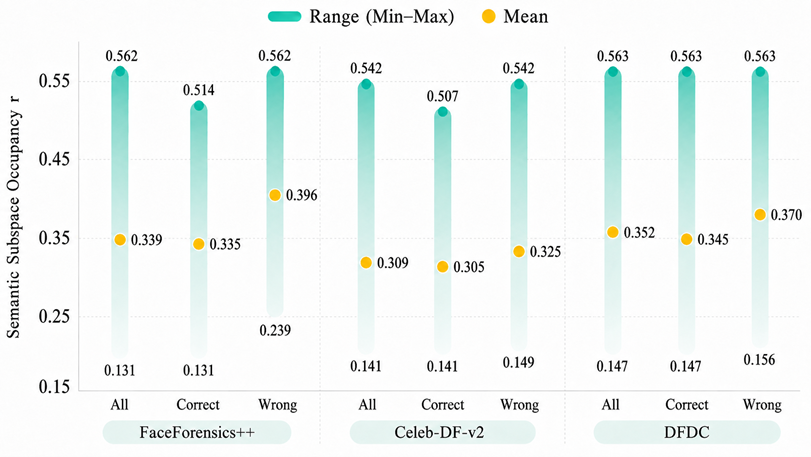}
    \caption{Semantic subspace occupancy of the fine-tuned CLIP on FaceForensics++, Celeb-DF v2, and DFDC. Across all datasets, wrongly classified samples exhibit consistently higher semantic subspace occupancy.}
    \label{sec3-fig2}
\end{figure}

\subsection{Semantic Fallback as a Cause of Cross-Domain Generalization Failure}
\subsubsection{How Semantic Fallback Relates to Prediction Results}
The analysis above establishes that semantic fallback persists as a stable property of fine-tuned representations. However, the existence of residual semantic structure alone does not yet imply that it actively impairs cross-domain detection. To examine whether prediction failures are more strongly associated with such residual semantic dominance, we introduce a sample-level metric termed semantic subspace occupancy.

\begin{definition}[Semantic Subspace Occupancy]
Let $\{\bm{s}_c\}_{c=1}^{C}$ denote the semantic features computed from the final-layer features of the frozen CLIP encoder on the FF++ training set, where $\bm{s}_c$ is obtained by averaging the frozen CLIP features of samples belonging to semantic group $c$. Inspired by between-class scatter analysis in Fisher-style discriminant methods~\cite{belhumeur1997eigenfaces}, we construct the between-semantics scatter matrix as
\[
\bm{S}_{b} \triangleq 
\sum_{c=1}^{C} n_c(\bm{s}_c-\bar{\bm{s}})(\bm{s}_c-\bar{\bm{s}})^\top,
\]
where $n_c$ is the number of samples in semantic group $c$ and 
$\bar{\bm{s}}=\frac{1}{N}\sum_{c=1}^{C}n_c\bm{s}_c$ is the global mean semantic feature. We perform eigen-decomposition on $\bm{S}_{b}$ and orthonormalize its top-$K$ eigenvectors to obtain a semantic basis $\bm{U}\in\mathbb{R}^{d\times K}$. Given a fine-tuned feature $\bm{f}_i$, its semantic subspace occupancy is defined as $r_i \triangleq 
\frac{\left\|\bm{U}\bm{U}^{\top}\bm{f}_i\right\|_2^2}
{\left\|\bm{f}_i\right\|_2^2}.$
\end{definition}

The semantic basis $\bm{U}$ captures the principal identity-related semantic directions inherited from CLIP pre-training, without involving any target-domain information. The occupancy score $r_i$ measures the proportion of feature energy lying inside this frozen semantic subspace. A larger $r_i$ indicates stronger alignment with the semantic geometry inherited from pre-training, suggesting that the feature remains more strongly constrained by the frozen semantic manifold. Compared with $\Delta S$, which captures pairwise identity-aware structure at the distribution level, $r_i$ provides a more direct sample-level indicator of semantic fallback. The results are shown in Figure~\ref{sec3-fig2}. For each dataset, we report the mean semantic subspace occupancy for all samples, correctly classified samples, and wrongly classified samples, together with the corresponding minimum and maximum values. A clear pattern emerges across all three datasets: wrongly classified samples consistently exhibit larger semantic subspace occupancy than correctly classified ones. This phenomenon is especially important in our context, because it suggests that \textit{failure cases are more strongly constrained by the semantic manifold inherited from the pre-trained model.} These findings provide more direct evidence for the semantic fallback hypothesis. When transferable forgery cues are insufficient, the detector does not fully reorganize the feature space around manipulation-specific evidence; instead, a larger portion of the feature representation remains anchored to the frozen semantic subspace. As a result, the model becomes less capable of separating real and fake samples based on universal forensic traces, and its generalization performance degrades accordingly. 

\subsubsection{How Semantic Removal Affects Detection Results}
We next perform a direct intervention to test whether the frozen semantic component indeed impairs transferable forensic discrimination. The key idea is simple: if the frozen semantic subspace acts as a source of semantic fallback, then removing this component from the probed representation should reduce identity-related interference and improve cross-domain fake/real discrimination. To this end, we use the semantic basis $\bm{U}$ estimated from the FF++ training set using the final-layer CLS features of the frozen CLIP encoder, as described in the previous subsection.

\begin{definition}[Controlled Semantic Projection Removal]
Given a semantic basis $\bm{U}\in\mathbb{R}^{d\times K}$ with orthonormal columns and a frozen feature $\bm{f}_i\in\mathbb{R}^{d}$, we define its partially de-semanticized representation as $
\bm{f}_i^{(\lambda)} \triangleq \bm{f}_i - \lambda \bm{U}\bm{U}^{\top}\bm{f}_i,$
where $\lambda \in [0,1]$ controls the degree of semantic subtraction.
\end{definition}
\noindent Here, $\bm{U}\bm{U}^{\top}\bm{f}_i$ denotes the component of $\bm{f}_i$ aligned with the frozen semantic basis, while $\bm{f}_i^{(\lambda)}$ retains the residual information after controlled removal of this component. When $\lambda=0$, the original frozen feature is preserved; when $\lambda=1$, the semantic projection is fully removed. By varying $\lambda$, we can examine in a continuous manner how semantic removal affects downstream detection performance.

 \begin{table}[t]
\centering
\caption{Cross-domain detection performance under different degrees of semantic removal. A linear probe is trained on features after partial semantic subtraction based on the frozen CLIP semantic basis. The results show that even simple semantic removal helps improve detection performance.}
\label{sec3-tab2}
\resizebox{\linewidth}{!}{%
\begin{tabular}{c|ccccc}
\toprule
Dataset & $\lambda=0$ & $\lambda=0.2$ & $\lambda=0.4$ & $\lambda=0.6$ & $\lambda=0.8$ \\
\midrule
FaceForensics++~\cite{faceforensics++} & 84.8 & 85.3 & 84.2 & 85.3 & 85.1 \\
Celeb-DF v2~\cite{celebdf}  & 66.4 & 67.4 & 66.9 & 66.8 & 67.8\\
DFDC~\cite{dfdc} & 65.7 & 67.0 & 66.5 & 65.8 & 66.2 \\
\bottomrule
\end{tabular}
}
\end{table}

Based on these transformed features, we train a linear probe on FF++ for two epochs and then evaluate it on the in-domain FF++ test set as well as two cross-domain benchmarks, Celeb-DF v2 and DFDC. The results are reported in Table~\ref{sec3-tab2}. Two observations can be drawn. First, semantic removal brings an immediate benefit to detection performance, especially in the cross-domain setting. Even a simple subtraction of the frozen semantic component leads to consistent gains on unseen datasets. For example, compared with $\lambda=0$, applying moderate semantic removal improves the performance on Celeb-DF v2 and DFDC, while keeping the in-domain FF++ performance largely stable. This result is important because it provides intervention-based evidence that the residual semantic component is not merely an accompanying property of the representation, but an active factor that interferes with transferable forensic discrimination. Second, the effect of semantic removal is not uniform across datasets: different target domains favor different subtraction strengths. Moderate semantic removal yields the best improvement, whereas stronger subtraction does not necessarily produce further gains. This suggests that the semantic and forensic components may not be perfectly disentangled in the learned representation. 

These results further support our central claim that semantic fallback constitutes an important cause of detection failure, particularly in cross-domain scenarios. More importantly, they provide a practical implication for the design of AI-generated image detectors based on powerful pre-trained vision foundation models (\textit{e.g.,} CLIP): \textit{the challenge is not whether to use strong pre-trained VFMs, but how to prevent their inherited semantic priors from dominating the forensic representation.} A robust detector should therefore avoid organizing its feature space primarily around high-level semantic content, and instead encourage alignment with transferable forgery traces that remain stable across semantic variations.

\subsection{A Representation View of Semantic Suppression}
\label{sec3-3}
The intervention results above suggest a central conclusion: although semantic suppression improves cross-domain detection, semantic information and forensic evidence are not perfectly disentangled in the learned representation, but instead remain partially coupled in feature space. This observation motivates the coupled-representation view developed below, which provides a principled explanation for why semantic suppression can improve generalization, and why its strength must be chosen with care.

\begin{assumption}[Coupled Representation]\label{ass:coupled}
Let $\bm{f} \in \mathbb{R}^{d}$ denote the learned feature representation of an input sample. We assume that $
\bm{f}
=
\bm{f}_{\mathrm{for}}
+
\bm{f}_{\mathrm{sem}}
+
\bm{f}_{\mathrm{shr}}
+
\bm{\varepsilon},
\label{eq:coupled_rep}$ where $\bm{f}_{\mathrm{for}}$ denotes a forensic-specific component that is predictive of the real/fake label and comparatively more stable across domains than semantic shortcut features, $\bm{f}_{\mathrm{sem}}$ denotes a semantic-dominant component that may induce domain-sensitive shortcuts, $\bm{f}_{\mathrm{shr}}$ denotes a shared component correlated with both semantic structure and task-relevant forensic evidence, and $\bm{\varepsilon}$ collects residual variation not captured by the preceding structured terms.
\end{assumption}

This decomposition is conceptual rather than uniquely estimated, and is used to characterize different roles of feature components under domain shift.
\begin{definition}[Semantic Suppression]\label{def:sso}
Let $\bm{P}_{\hat{\mathcal S}}=\bm{U}\bm{U}^{\top}$ denote the orthogonal projector onto an estimated semantic-dominant subspace $\hat{\mathcal S}$, where $\bm{U}$ contains an orthonormal basis of $\hat{\mathcal S}$. For a suppression strength $\lambda \in [0,1]$, define $
T_{\lambda}(\bm{f})
=
\left(
\bm{I}
-
\lambda \bm{P}_{\hat{\mathcal S}}
\right)
\bm{f}.
\label{eq:semantic_suppression}$
Here, $\hat{\mathcal S}$ is an estimated semantic-dominant subspace and may deviate from the ideal semantic subspace $\mathcal S^{\star}$ due to finite-sample estimation error, domain shift, or imperfect semantic supervision. Applying $T_{\lambda}$ to Eq.~\eqref{eq:coupled_rep} yields
\begin{align}
T_{\lambda}(\bm{f})
&=
\left(
\bm{I}
-
\lambda \bm{P}_{\hat{\mathcal S}}
\right)
\bm{f}_{\mathrm{for}}
+
\left(
\bm{I}
-
\lambda \bm{P}_{\hat{\mathcal S}}
\right)
\bm{f}_{\mathrm{sem}}
\nonumber\\
&\quad
+
\left(
\bm{I}
-
\lambda \bm{P}_{\hat{\mathcal S}}
\right)
\bm{f}_{\mathrm{shr}}
+
\left(
\bm{I}
-
\lambda \bm{P}_{\hat{\mathcal S}}
\right)
\bm{\varepsilon}.
\label{eq:suppressed_rep}
\end{align}
\end{definition}

\begin{proposition}[Margin Shift under Semantic Suppression]\label{prop:margin}
Let $h(\bm{f})=\bm{w}^{\top}\bm{f}$ be a linear classifier, and let
$m_{\lambda}=y\,\bm{w}^{\top}T_{\lambda}(\bm{f})$ denote the signed margin under semantic suppression, where $y\in\{-1,+1\}$ is the ground-truth label. Under Assumption~\ref{ass:coupled}, the margin admits the following decomposition:
\begin{align}
m_{\lambda}
&=
y\,\bm{w}^{\top}
\left(
\bm{f}_{\mathrm{for}}
+
\bm{f}_{\mathrm{sem}}
+
\bm{f}_{\mathrm{shr}}
+
\bm{\varepsilon}
\right)
\nonumber\\
&\quad
-
\lambda y\,\bm{w}^{\top}
\bm{P}_{\hat{\mathcal S}}
\left(
\bm{f}_{\mathrm{for}}
+
\bm{f}_{\mathrm{sem}}
+
\bm{f}_{\mathrm{shr}}
+
\bm{\varepsilon}
\right).
\label{eq:margin_decomp}
\end{align}
Equivalently,
\begin{equation}
m_{\lambda}-m_{0}
=
-\lambda y\,\bm{w}^{\top}
\bm{P}_{\hat{\mathcal S}}
\left(
\bm{f}_{\mathrm{for}}
+
\bm{f}_{\mathrm{sem}}
+
\bm{f}_{\mathrm{shr}}
+
\bm{\varepsilon}
\right),
\label{eq:margin_shift}
\end{equation}
where $m_0 = y\,\bm{w}^{\top}\bm{f}$ is the unsuppressed margin.
\end{proposition}

\noindent This result follows directly from the linearity of $T_{\lambda}$. Proposition~\ref{prop:margin} shows that semantic suppression does not only attenuate the semantic-dominant contribution $\bm{P}_{\hat{\mathcal S}}\bm{f}_{\mathrm{sem}}$, but may also remove the shared component $\bm{P}_{\hat{\mathcal S}}\bm{f}_{\mathrm{shr}}$ and useful forensic signal if $\hat{\mathcal S}$ is imperfectly estimated. Thus, improved transfer reflects a trade-off between eliminating unstable semantic shortcuts and preserving discriminative evidence. To formalize this trade-off at the target-risk level, we consider the following local quadratic approximation.

\begin{assumption}[Target-Risk Trade-off]\label{ass:risk}
Let $\mathcal{R}_{\mathrm{tgt}}(\lambda)$ denote the target-domain risk under the suppressed representation $T_{\lambda}(\bm{f})$. Let $\bm{P}_{\mathcal S^\star}$ denote the ideal semantic projection operator, and let $\bm{P}_{\hat{\mathcal S}}$ denote its estimated counterpart. In the operating regime of interest, we assume that the target-domain risk admits the following second-order surrogate:
\begin{equation}
\mathcal{R}_{\mathrm{tgt}}(\lambda)
\approx
\mathcal{R}_{0}
+
(1-\lambda)^2 \Gamma_{\mathrm{sem}}
+
\lambda^2 \Gamma_{\mathrm{share}}
+
\lambda^2 \Gamma_{\mathrm{mis}},
\label{eq:risk_tradeoff}
\end{equation}
where
\begin{align}
\Gamma_{\mathrm{sem}}
&\triangleq 
\mathbb{E}_{(\bm{x},y)\sim\mathcal{D}_{\mathrm{tgt}}}
\left[
\omega_{\mathrm{sem}}(\bm{x},y)
\left\|
\bm{P}_{\mathcal S^\star}
\bm{f}_{\mathrm{sem}}(\bm{x})
\right\|_2^2
\right] > 0,
\label{eq:gamma_sem}
\\
\Gamma_{\mathrm{share}}
&\triangleq 
\mathbb{E}_{(\bm{x},y)\sim\mathcal{D}_{\mathrm{tgt}}}
\left[
\omega_{\mathrm{share}}(\bm{x},y)
\left\|
\bm{P}_{\hat{\mathcal S}}
\bm{f}_{\mathrm{shr}}(\bm{x})
\right\|_2^2
\right] > 0,
\label{eq:gamma_share}
\\
\Gamma_{\mathrm{mis}}
&\triangleq 
\mathbb{E}_{(\bm{x},y)\sim\mathcal{D}_{\mathrm{tgt}}}
\left[
\omega_{\mathrm{mis}}(\bm{x},y)
\left\|
\left(
\bm{P}_{\hat{\mathcal S}}
-
\bm{P}_{\mathcal S^\star}
\right)
\bm{f}(\bm{x})
\right\|_2^2
\right] \ge 0.
\label{eq:gamma_mis}
\end{align}
\end{assumption}

Here, $\Gamma_{\mathrm{sem}}$ represents the effective penalty of retaining semantic shortcuts, $\Gamma_{\mathrm{share}}$ represents the cost of suppressing shared task-relevant evidence, and $\Gamma_{\mathrm{mis}}$ represents the cost induced by subspace mismatch. The weights $\omega_{\mathrm{sem}}$, $\omega_{\mathrm{share}}$, and $\omega_{\mathrm{mis}}$ absorb local loss curvature and sample-dependent sensitivity. Eq.~\eqref{eq:risk_tradeoff} is a local second-order surrogate motivated by Eq.~\eqref{eq:margin_shift}, rather than a direct algebraic consequence. $\mathcal{R}_{0}$ denotes a residual base risk and should not be interpreted as $\mathcal{R}_{\mathrm{tgt}}(0)$. The role of Eq.~\eqref{eq:risk_tradeoff} is to isolate the leading-order benefit of attenuating semantic shortcuts from the leading-order cost of suppressing shared or misidentified task-relevant information.

\begin{proposition}[Optimal Suppression Strength]\label{prop:lambda_star}
Under Assumption~\ref{ass:risk}, the approximate target risk in Eq.~\eqref{eq:risk_tradeoff} is uniquely minimized at $\lambda^{\star}
=
\frac{\Gamma_{\mathrm{sem}}}
{
\Gamma_{\mathrm{sem}}
+
\Gamma_{\mathrm{share}}
+
\Gamma_{\mathrm{mis}}
}.
\label{eq:lambda_star}$
Moreover, $\lambda^{\star}\in(0,1)$ whenever $\Gamma_{\mathrm{sem}}>0$ and 
$\Gamma_{\mathrm{share}}+\Gamma_{\mathrm{mis}}>0$.
\end{proposition}

\begin{proof}
Let $\Gamma_{\mathrm{cost}}=\Gamma_{\mathrm{share}}+\Gamma_{\mathrm{mis}}$. Ignoring the constant $\mathcal{R}_0$, Eq.~\eqref{eq:risk_tradeoff} becomes
\begin{equation}
\widetilde{\mathcal{R}}_{\mathrm{tgt}}(\lambda)
=
(1-\lambda)^2\Gamma_{\mathrm{sem}}
+
\lambda^2\Gamma_{\mathrm{cost}}.
\end{equation}
Taking the derivative with respect to $\lambda$ and setting it to zero gives $
-2(1-\lambda)\Gamma_{\mathrm{sem}}
+
2\lambda\Gamma_{\mathrm{cost}}
=
0,$
which yields the expression for $\lambda^{\star}$. Since the second derivative is
$2(\Gamma_{\mathrm{sem}}+\Gamma_{\mathrm{cost}})>0$, the minimizer is unique. The condition $\lambda^\star\in(0,1)$ follows directly from $\Gamma_{\mathrm{sem}}>0$ and $\Gamma_{\mathrm{cost}}>0$.
\end{proof}

Proposition~\ref{prop:lambda_star} formalizes the central implication of the coupled-representation view: semantic suppression improves transfer when it removes target-unstable semantic shortcuts, but overly aggressive suppression becomes harmful once it erases shared or misidentified task-relevant information. This decomposition therefore suggests the existence of a non-trivial optimal suppression strength, denoted by $\lambda^\star$, that balances the removal of semantic nuisance against the preservation of shared discriminative evidence.

\section{Methodology}
\begin{figure*}[t]
\centering
\includegraphics[width=0.97\linewidth]{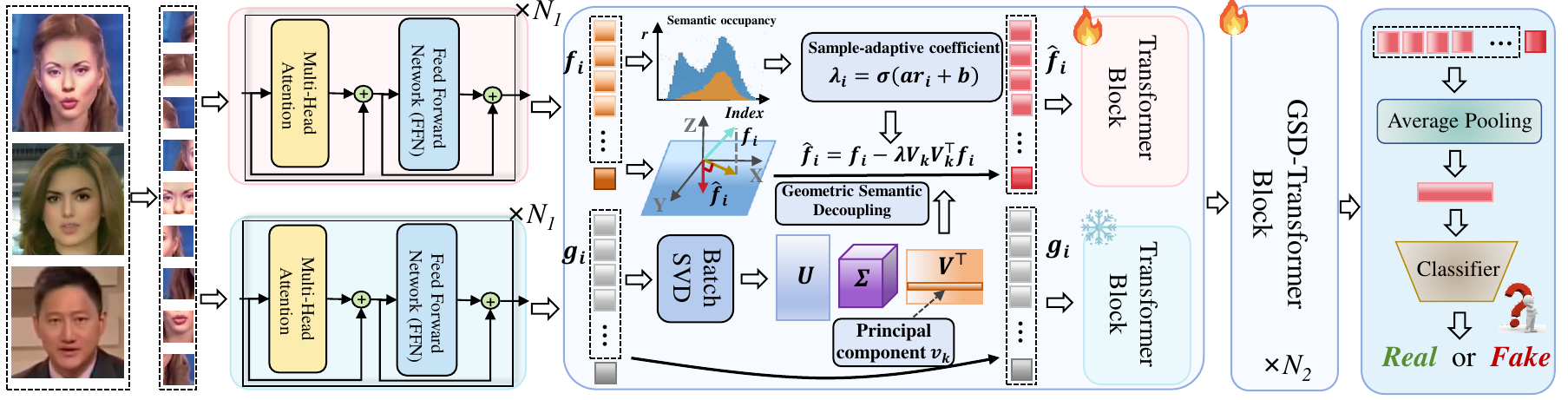}
\caption{Overview of the proposed detection framework with Geometric Semantic Decoupling (GSD). During training, a trainable forensic branch and a frozen semantic anchor branch, both initialized from the same CLIP backbone, are processed in parallel. At selected transformer layers, the frozen branch estimates semantically dominant directions via Batch-SVD, and the GSD module adaptively suppresses the corresponding components in the trainable features to produce semantically decoupled representations.}
\label{sec4-fig1}
\end{figure*}

The analyses in Section~\ref{sec3} reveal that detectors built upon powerful pre-trained vision-language models exhibit a \emph{semantic fallback}: under domain shift, their representations remain strongly biased toward the semantic structure inherited from pre-training, which interferes with transferable forensic discrimination. And thus, we introduce a \textbf{Geometric Semantic Decoupling (GSD)} module. Specifically, since the undesired bias originates from the pre-trained semantic subspace, we leverage a frozen CLIP encoder to provide a stable estimate of this semantic structure, and explicitly suppress its contribution in the trainable representation. In this way, the detector is encouraged to reorganize its feature space around manipulation-related evidence rather than high-level semantic content.

\subsection{Overview}
\label{sec4-1}
The overall framework is illustrated in Figure~\ref{sec4-fig1}. Our detector consists of a trainable forensic branch and a frozen semantic anchor branch, both initialized from the same pre-trained CLIP backbone. The frozen branch is used to preserve the semantic geometry inherited from pre-training, while the trainable branch learns features for real/fake discrimination. At selected transformer layers, the proposed GSD module estimates semantically dominant directions from the frozen branch and suppresses the corresponding components in the trainable features, yielding a semantically decoupled representation for downstream detection.

A key property of GSD is that the suppression strength is adaptive rather than fixed, governed by the semantic dominance of each sample and modulated by learnable layer-specific parameters. This design allows the model to apply stronger decoupling when semantic interference is severe and milder correction otherwise. Building on this adaptive mechanism, we first consider an ideal formulation where the semantic subspace is estimated at the instance level. However, such per-sample SVD estimation introduces substantial overhead in practice, so we derive Batch-SVD as an efficient batch-level approximation, which amortizes semantic subspace estimation across samples and enables scalable optimization. We further develop three inference strategies. Batch Inference targets large-scale evaluation, where Batch-SVD is directly applied to each mini-batch for efficiency. Per-sample Inference aligns with online detection scenarios, using Single-SVD to construct a sample-specific semantic subspace when inputs must be processed independently. Finally, to mitigate the cost of per-sample estimation, we introduce Reference-based Inference, which approximates the semantic subspace using a reusable semantic reference set. Importantly, both theoretical analysis and empirical results show that all three strategies induce closely aligned forgery-oriented feature manifolds, ensuring that practical approximations preserve the core effect of semantic decoupling.


\subsection{Geometric Semantic Decoupling with Single-Sample SVD}
\label{sec:method_single_svd}

The analysis in Section~\ref{sec3} shows that semantically dominant components can overshadow forgery-specific cues and degrade cross-domain discrimination. This observation suggests that improving generalization requires not only adapting a powerful vision foundation model to the detection task, but also explicitly suppressing the semantic directions that dominate the learned representation. To this end, the key question becomes how to identify the semantically dominant directions in a structured and data-dependent manner. Our solution is to estimate such directions directly from the feature geometry of a frozen semantic reference model. Specifically, we introduce a frozen CLIP encoder as a \emph{semantic anchor}. The motivation is that large-scale vision-language pre-training endows CLIP with rich and stable semantic structure, including high-level content variations related to identity, object category, and scene semantics. Since this frozen encoder is not affected by downstream forensic fine-tuning, it preserves the original semantic geometry of the pre-trained model and therefore provides a suitable reference for estimating semantic-dominated directions.

Given this semantic anchor, we further introduce singular value decomposition (SVD) as a geometric tool to extract the principal directions of semantic variation. For the $i$-th sample at transformer layer $l$, let $\bm{f}_{i,l} \in \mathbb{R}^{(N+1)\times d}$ denote the token representation from the trainable forensic branch, where $N$ is the number of patch tokens and $d$ is the hidden dimension. Let $\bm{g}_{i,l} \in \mathbb{R}^{(N+1)\times d}$ denote the corresponding representation from the frozen CLIP anchor. Since patch tokens retain richer spatial semantic structure than the global CLS token, we use the non-CLS tokens of the frozen branch to estimate semantic directions. Denoting the non-CLS token matrix by $\bar{\bm{P}}_{i,l} \in \mathbb{R}^{N\times d}$, we perform singular value decomposition as $
\bar{\bm{P}}_{i,l} = \bm{U}_{i,l}\bm{\Sigma}_{i,l}\bm{V}_{i,l}^{\top}.$ The top-$k$ right singular vectors, denoted by $\bm{V}^{(k)}_{i,l}$, capture the dominant directions of variation in the frozen semantic representation. In our framework, these directions serve as an estimate of the semantically dominant subspace for the current sample.

To determine the appropriate strength of semantic suppression, we build on the empirical observation in Section~\ref{sec3} that wrongly classified samples consistently exhibit higher semantic-subspace occupancy than correctly classified ones. This observation suggests that the extent to which a sample remains confined to the frozen semantic manifold is closely related to its detection difficulty. Therefore, rather than applying a uniform correction to all samples, we introduce an explicit measure of semantic dominance to quantify how much of the current trainable representation lies in the semantic subspace estimated from the frozen anchor. Specifically, given the semantic basis $\bm{V}^{(k)}_{i,l}$, we define the semantic occupancy of the $i$-th sample at layer $l$ as
\begin{equation}
r_{i,l}
=
\frac{
\|\overbrace{\bm{f}_{i,l}\bm{V}_{i,l}^{(k)}(\bm{V}_{i,l}^{(k)})^{\top}}^{\bm{f}_{i,l,\parallel}}\|_F^2
}{
\|\bm{f}_{i,l}\|_F^2
}.
\end{equation}
where $\bm{f}_{i,l,\parallel}$ denotes the semantic component of the trainable representation, obtained by projecting $\bm{f}_{i,l}$ onto the dominant directions estimated from the frozen CLIP anchor. A larger $r_{i,l}$ indicates that more representation energy lies in the semantic subspace, corresponding to a stronger semantic fallback effect.

The empirical findings in Section~\ref{sec3} further suggest that semantic suppression should not be controlled by a fixed subtraction ratio. On the one hand, Table~\ref{sec3-tab2} shows that different degrees of semantic removal lead to different detection performance, indicating that semantic decoupling is beneficial only when applied with an appropriate strength. On the other hand, Table~\ref{sec3-tab1} shows that the strength of semantic coupling varies across layers, especially near the top of the encoder. These observations suggest that the subtraction strength should be adaptive to both the current sample and the network depth. We therefore adopt a layer-aware and sample-adaptive subtraction coefficient. The final semantically decoupled representation is then obtained by
\begin{equation}
\hat{\bm{f}}_{i,l}
=
\bm{f}_{i,l}
-
\underbrace{
\lambda_{\max}\cdot
\sigma\!\left(
a_l r_{i,l}+b_l
\right)
}_{\lambda_{i,l}}
\bm{f}_{i,l,\parallel},
\end{equation}
where $\sigma(\cdot)$ is the sigmoid function, $\lambda_{\max}$ denotes the maximum subtraction strength, and $a_l$ and $b_l$ are learnable parameters associated with layer $l$. In this way, samples with higher semantic occupancy receive stronger suppression, while different layers are allowed to calibrate the degree of semantic removal according to their semantic sensitivity.

This formulation can be understood as a geometry-guided adaptive correction: the trainable representation is explicitly pushed away from semantically dominant directions inherited from the frozen foundation model, while the amount of correction is jointly modulated by sample-specific semantic occupancy and layer-specific semantic sensitivity. After applying GSD at the selected transformer layers, the detector is optimized end-to-end using the standard classification objective:
\begin{equation}
\mathcal{L}_{\mathrm{cls}}
=
-\frac{1}{B}
\sum_{i=1}^{B}
\log
\frac{
\exp\left(\bm{w}_{t_i}^{\top}\hat{\bm{f}}_i + b_{t_i}\right)
}{
\sum_{c=0}^{1}
\exp\left(\bm{w}_{c}^{\top}\hat{\bm{f}}_i + b_{c}\right)
},
\label{eq:cls_loss}
\end{equation}
where $\hat{\bm{f}}_i$ denotes the representation of the $i$-th sample,
$t_i\in\{0,1\}$ is its ground-truth class label, and $\bm{w}_c$ and $b_c$ denote
the classifier weight and bias for class $c$, respectively.

\subsection{Batch-SVD Approximation for Efficient Training}
\begin{table}[t]
\centering
\renewcommand{\arraystretch}{0.9}
\caption{Per-image runtime comparison of Batch-SVD and Single-SVD.
The $ratio$ denotes the percentage of total computation time spent on SVD operations. Batch-SVD significantly reduces the per-image cost compared with Single-SVD for both training and inference.}
\label{sec4-tab1}
\resizebox{\linewidth}{!}{
\begin{tabular}{lccc}
\toprule
Setting & SVD Time (ms) & Ratio (\%) & Speedup \\
\midrule
\multicolumn{4}{c}{\textbf{Training (per image)}} \\
\midrule
Batch-SVD & $5.43 \pm 0.02$ & $32.72 \pm 0.09$ & $\sim$15.6$\times$ \\
Single-SVD & $84.81 \pm 0.02$ & $88.42 \pm 0.07$ & 1$\times$ \\
\midrule
\multicolumn{4}{c}{\textbf{Inference (per image)}} \\
\midrule
Batch-SVD & $3.37 \pm 0.02$ & $50.45 \pm 0.05$ & $\sim$15.2$\times$ \\
Single-SVD & $51.13 \pm 0.03$ & $93.90 \pm 0.06$ & 1$\times$ \\
\bottomrule
\end{tabular}
}
\end{table}

The Single-SVD formulation in Section~\ref{sec:method_single_svd} provides a direct way to estimate semantically dominant directions from the frozen CLIP anchor and to suppress their influence on the trainable forensic representation. However, applying a separate SVD to every sample at every selected transformer layer is computationally expensive, especially for large backbones such as ViT-L/14, where GSD is inserted into multiple layers and optimized end-to-end. As shown in Table~\ref{sec4-tab1}, the SVD operation alone accounts for over $88\%$ of the total training time and over $93\%$ of the total inference time under the Single-SVD formulation, making it impractical for efficient optimization. To make GSD computationally feasible, we introduce a Batch-SVD approximation that estimates the semantic subspace jointly at the mini-batch level. The central idea is to replace repeated per-sample decompositions with a single decomposition over the aggregated frozen representations in the current batch. In this way, semantically dominant directions are still estimated from the frozen anchor branch, but the cost of subspace estimation is amortized across all samples in the mini-batch.

Specifically, at transformer layer $l$, consider a mini-batch of $B$ samples. For the $i$-th sample, let $\bar{\bm{P}}_{i,l} \in \mathbb{R}^{N\times d}$ denote the non-CLS token matrix extracted from the frozen semantic anchor.  We concatenate all frozen patch-token matrices along the token dimension and perform a single singular value decomposition:
\begin{equation}
\bar{\bm{M}}_l
=
\begin{bmatrix}
\bar{\bm{P}}_{1,l} \\
\bar{\bm{P}}_{2,l} \\
\vdots \\
\bar{\bm{P}}_{B,l}
\end{bmatrix}
=
\bm{U}^{(b)}_l\bm{\Sigma}^{(b)}_l(\bm{V}^{(b)}_l)^\top,
\qquad
\bar{\bm{M}}_l \in \mathbb{R}^{BN\times d}.
\end{equation}
We then use the top-$k$ right singular vectors of $\bm{V}^{(b)}_l$ to form a shared semantic basis for the entire mini-batch at layer $l$, denoted by $\bm{V}^{(b,k)}_l \in \mathbb{R}^{d\times k}$.

Given this batch-level semantic basis, the semantic occupancy and the adaptive decoupling coefficient for each sample are computed in the same manner as in Section~\ref{sec:method_single_svd}, except that the sample-specific basis $\bm{V}^{(k)}_{i,l}$ is replaced by the shared batch-level basis $\bm{V}^{(b,k)}_l$. Accordingly, the final semantically decoupled representation of the $i$-th sample is given by $
\tilde{\bm{f}}_{i,l}
=
\bm{f}_{i,l}
-
\lambda_{i,l}\bm{f}_{i,l,\parallel},$ where $\bm{f}_{i,l,\parallel}$ denotes the projection of $\bm{f}_{i,l}$ onto the shared semantic subspace spanned by $\bm{V}^{(b,k)}_l$. This approximation preserves the essential mechanism of GSD in two aspects. First, the semantic directions are still estimated from the frozen CLIP anchor rather than from the trainable branch itself, so the decoupling process continues to be guided by the original semantic geometry inherited from pre-training. Second, although the semantic basis is shared within each mini-batch, the actual subtraction strength remains sample-adaptive through $r_{i,l}$, which allows different samples to receive different levels of semantic suppression according to their own semantic dominance. In this sense, Batch-SVD approximates the semantic subspace estimation step, but preserves the core geometry-guided and sample-aware nature of GSD.

The practical advantage of Batch-SVD is substantial. Under a practical setting of 10,000 images with a batch size of 128 on a single NVIDIA A100 (80GB) GPU, with GSD inserted into the last four transformer layers, replacing repeated per-sample decompositions with a single decomposition per layer and mini-batch significantly reduces the computational burden of semantic subspace estimation. As shown in Table~\ref{sec4-tab1}, Batch-SVD reduces the per-image SVD cost by more than $15\times$ compared with Single-SVD in both training and inference, while lowering the proportion of total computation spent on SVD from $88.42\%$ to $32.72\%$ during training and from $93.90\%$ to $50.45\%$ during inference. These results show that Batch-SVD provides an effective efficiency--accuracy trade-off, making GSD practical for large-scale optimization without changing its semantic-anchor-based decoupling mechanism.

\begin{figure*}[t]
\centering
\includegraphics[width=0.97\linewidth]{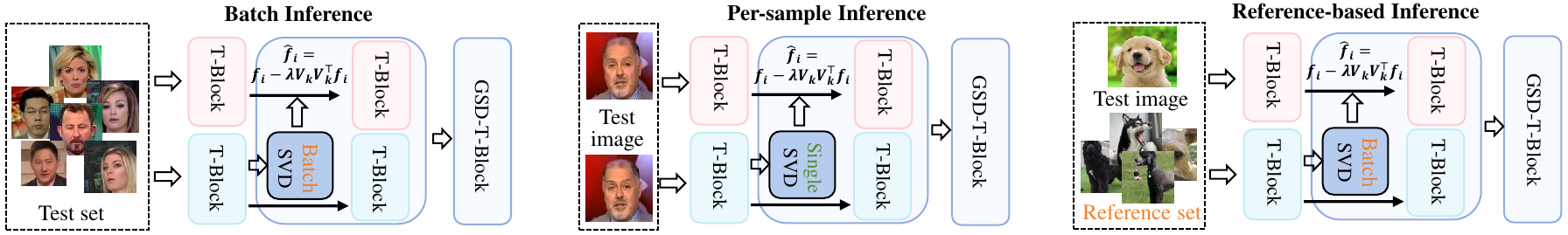}
\caption{During inference, the trained detector can be deployed with three strategies: \emph{Batch Inference} for efficient dataset-level evaluation, \emph{Per-sample Inference} for online image-level detection with sample-specific semantic estimation, and \emph{Reference-based Inference} for a more efficient approximation based on a semantic reference set.}
\label{sec4-fig2}
\end{figure*}

\subsection{Inference and Deployment}
\label{sec:inference}
\subsubsection{Inference Protocol in Real-world Detection}
In Figure~\ref{sec4-fig2}, we consider three inference protocols for GSD, designed for two common scenarios: \emph{dataset-level benchmark evaluation} and \emph{practical online deployment}. In benchmark settings, test images are typically processed in mini-batches, which naturally supports batch-wise semantic estimation. In contrast, in real-world deployment, inputs often arrive sequentially in a streaming manner, requiring each image to be processed independently. To accommodate both scenarios while balancing accuracy and efficiency, we consider \emph{Batch Inference}, \emph{Per-sample Inference}, and \emph{Reference-based Inference}.

\textbf{Batch Inference.}
For standard benchmark evaluation on datasets such as FF++, Celeb-DF v2, and DFDC, we directly adopt the Batch-SVD formulation used during training. At each selected layer, the frozen semantic anchor estimates a shared semantic basis from the current mini-batch, and each sample is then decoupled using its own adaptive suppression coefficient. This protocol is computationally efficient and well aligned with large-scale dataset-level evaluation.

\textbf{Per-sample Inference.}
For online or streaming detection, batch-level semantic estimation may be unavailable, since each input arrives independently. In this case, we use the Single-SVD formulation introduced in Section~\ref{sec:method_single_svd}. For each test image, the frozen anchor estimates a sample-specific semantic subspace from its token embeddings, and the trainable representation is decoupled accordingly. This protocol is conceptually faithful to the original design and is particularly suitable for image-level deployment.

\textbf{Reference-based Inference.}
Although Per-sample Inference provides precise sample-wise semantic estimation, it remains computationally expensive because it requires an SVD operation for every test image, as shown in Table~\ref{sec4-tab1}. To reduce this cost, we further introduce a reference-based approximation strategy. Specifically, we construct a semantic reference set and use the frozen CLIP features of these reference samples to estimate a shared semantic subspace. During inference, instead of performing SVD for each input, we project the test representation onto this reference subspace and suppress the semantic-dominant components accordingly. This strategy is particularly attractive in practical scenarios where semantically similar reference samples can be retrieved or prepared in advance.



\subsubsection{Theoretical Justification of Inference Compatibility}
To theoretically justify the compatibility of these inference protocols with a detector trained under Batch-SVD, we analyze their differences through the lens of subspace perturbation. The effect of switching the inference protocol can be reduced to the discrepancy between the corresponding estimated semantic subspaces. When these subspaces remain sufficiently close, the resulting decoupled representations and detection scores are accordingly well aligned.

\begin{assumption}[Shared Semantic Geometry]\label{ass:shared_geometry}
Let $P^{\star}$ denote the latent projector onto the semantic-dominant subspace induced by the frozen CLIP anchor. We assume that the projectors estimated by Batch Inference, Per-sample Inference, and Reference-based Inference, denoted by $P_{\mathit{batch}}$, $P_{\mathit{single}}$, and $P_{\mathit{ref}}$, respectively, are all perturbations of the same latent semantic geometry. Specifically, they satisfy $
\|P_{\alpha}-P^{\star}\| \leq \delta_{\alpha},
\alpha \in \{\mathit{batch},\mathit{single},\mathit{ref}\},
\label{eq:shared_geometry}$
for some bounded perturbation levels $\delta_{\alpha}$.
\end{assumption}

\begin{proposition}[Inference Gap under Projector Perturbation]\label{prop:inference_gap}
Let $T_{\lambda}^{(\alpha)}(\bm{f})
=
(I-\lambda P_{\alpha})\bm{f}
\label{eq:protocol_transform}$
denote the decoupled feature produced by inference protocol 
$\alpha \in \{\mathit{batch},\mathit{single},\mathit{ref}\}$.
Under Assumption~\ref{ass:shared_geometry}, for any two protocols $\alpha$ and $\beta$, their estimated projectors satisfy $\|P_{\alpha}-P_{\beta}\|
\le
\delta_{\alpha}+\delta_{\beta}.
\label{eq:projector_gap}$
Consequently, the resulting feature discrepancy is bounded by
\begin{equation}
\bigl\|
T_{\lambda}^{(\alpha)}(\bm{f})
-
T_{\lambda}^{(\beta)}(\bm{f})
\bigr\|
\le
\lambda
(\delta_{\alpha}+\delta_{\beta})
\|\bm{f}\|.
\label{eq:feature_gap}
\end{equation}
Moreover, if the downstream score function $h$ is $L$-Lipschitz, then the corresponding prediction-score discrepancy satisfies
\begin{equation}
\bigl|
h(T_{\lambda}^{(\alpha)}(\bm{f}))
-
h(T_{\lambda}^{(\beta)}(\bm{f}))
\bigr|
\le
L\lambda
(\delta_{\alpha}+\delta_{\beta})
\|\bm{f}\|.
\label{eq:score_gap}
\end{equation}
\end{proposition}

\begin{proof}
By Assumption~\ref{ass:shared_geometry} and the triangle inequality, we have
\begin{align}
\|P_{\alpha}-P_{\beta}\|
&\le
\|P_{\alpha}-P^{\star}\|
+
\|P_{\beta}-P^{\star}\|
\nonumber \le
\delta_{\alpha}+\delta_{\beta}.
\end{align}
Moreover, by the definition of $T_{\lambda}^{(\alpha)}$, 
\begin{align}
T_{\lambda}^{(\alpha)}(\bm{f})
-
T_{\lambda}^{(\beta)}(\bm{f})
&=
(I-\lambda P_{\alpha})\bm{f}
-
(I-\lambda P_{\beta})\bm{f}
\nonumber\\
&=
\lambda(P_{\beta}-P_{\alpha})\bm{f}.
\end{align}
Taking norms on both sides gives
\begin{align}
\bigl\|
T_{\lambda}^{(\alpha)}(\bm{f})
-
T_{\lambda}^{(\beta)}(\bm{f})
\bigr\|
&\le
\lambda
\|P_{\alpha}-P_{\beta}\|
\|\bm{f}\|
\nonumber \le
\lambda
(\delta_{\alpha}+\delta_{\beta})
\|\bm{f}\|.
\end{align}
The prediction-score bound follows directly from the $L$-Lipschitz continuity of the downstream score function $h$.
\end{proof}

Proposition~\ref{prop:inference_gap} shows that the discrepancy across inference protocols is bounded by the variation in the corresponding estimated semantic projectors. From this perspective, Batch-SVD can be viewed as a training-time approximation to semantic subspace estimation, rather than as a protocol that imposes a fundamentally different representation objective. During training, it encourages the trainable branch to suppress semantically dominant directions inherited from the frozen CLIP anchor, thereby shaping a more forgery-oriented feature space. Batch Inference, Per-sample Inference, and Reference-based Inference differ mainly in how the semantic subspace is estimated, while remaining grounded in the same anchor-induced semantic geometry. At inference time, both Per-sample Inference and Reference-based Inference apply the same core operation as Batch Inference, namely attenuating semantic-dominant components under this shared geometry. Therefore, when the estimated projectors remain close, the resulting decoupled features stay close to those produced by Batch Inference.

\section{Experiments}
\subsection{Implementation Details} 

\begin{table*}[t]
\centering
\renewcommand{\arraystretch}{0.9}
\caption{Summary of datasets and evaluation protocols used in this work.}
\label{sec6-tab1}
\resizebox{\textwidth}{!}{
\begin{tabular}{l l c c}
\toprule
\textbf{Regime} & \textbf{Protocol} & \textbf{Training Set} & \textbf{Test Set(s)} \\
\midrule
\multirow{2}{*}{Face forgery detection}
& Cross-dataset generalization
& FF++ (c23)
& Celeb-DF-v1 (CD1), Celeb-DF-v2 (CD2), DFD, DFDC, DFDCP \\

& Cross-manipulation generalization
& FF++ (c23)
& DF40: UniFace, e4s, FaceDancer, FSGAN, InSwap, SimSwap\\
\midrule
\multirow{2}{*}{Synthetic image detection}
& Cross-generator generalization
& ProGAN
& UniversalFakeDetect: 19 unseen test sets (GANs/Diffusions/Deepfakes) \\

& Cross-generator generalization
& SD v1.4
& GenImage: Midjourney, SD v1.5, ADM, GLIDE, Wukong, VQDM, BigGAN \\
\bottomrule
\end{tabular}}
\end{table*}

We adopt the pre-trained CLIP ViT-L/14~\cite{clip} as the visual backbone and optimize the entire network end-to-end using AdamW~\cite{adamw}. The backbone learning rate is set to $1\times10^{-6}$, and the default batch sizes are 128 for training and 512 for test. All experiments are conducted on NVIDIA A100 GPUs. To better simulate real-world degradations, we apply standard data augmentations, including Gaussian blur and JPEG compression, following~\cite{lsda,deepfakebench}. For the proposed Geometric Semantic Decoupling (GSD) module, we insert it into the last four layers of the CLIP encoder by default. Unless otherwise specified, the learnable modulation parameters are initialized with $a=2$ and $b=1$, the maximum subtraction strength is set to $\lambda_{\max}=1$, and the number of singular vectors used for semantic subspace estimation is set to $k=64$. During inference, we adopt Batch-Inference by default for efficiency. The effects of the above design choices, including inference strategy, insertion depth, SVD configuration, and GSD module parameters  are analyzed in detail in Subsections~\ref{sec6-3} and~\ref{sec6-4}.

We evaluate our method under two complementary regimes, \emph{i.e.}, face forgery detection and synthetic image detection, as summarized in Table~\ref{sec6-tab1}. For face forgery detection, following DeepfakeBench~\cite{deepfakebench}, we use FaceForensics++ (FF++)~\cite{faceforensics++} with the high-quality compression setting (c23) for training, and assess both cross-dataset and cross-manipulation generalization. To evaluate performance against recent \textit{GAN}-based and diffusion-based generative models, we further benchmark our method on UniversalFakeDetect~\cite{univfd} and GenImage~\cite{genimage} under the standard cross-generator generalization protocols~\cite{univfd,effort,c2p}.

To provide a comprehensive comparison, we benchmark against a broad spectrum of competitive detectors from three categories: \emph{(i) foundation/backbone baselines}, \emph{(ii) face-forgery detectors}, and \emph{(iii) synthetic-image detectors}. The first category includes generic visual backbones and foundation models, such as Efficient-B4~\cite{efficientnet}, ResNet-50~\cite{resnet}, DeiT-S~\cite{deit-s}, Swin-T~\cite{swint}, and CLIP~\cite{clip}. The second category consists of representative face forgery detectors, including F3Net~\cite{f3net}, Face X-Ray~\cite{face-x-ray}, RECCE~\cite{recce}, SBIs~\cite{sbi}, LSDA~\cite{lsda}, FCG~\cite{fcg}, ProDet~\cite{prodet}, VbSaT~\cite{vbsat}, DeepShield~\cite{deepshield}, $\chi^2$-DFD~\cite{x2dfd}, CDFA~\cite{cdfa}, Effort~\cite{effort}, FreqDebias~\cite{freqdebias}, LaV~\cite{lav}, CFM~\cite{cfm}, IID~\cite{iid}, UCF~\cite{ucf}, FakeRadar~\cite{fakeradar}, ExposDe~\cite{exposde}, FSgD~\cite{fsgd}, RAID~\cite{raid}, UDD~\cite{udd}, DFFake~\cite{dffake}, and ICT~\cite{ict}. The third category covers detectors designed for AI-generated imagery and general image forensics, including CNN-Spot~\cite{cnnspot}, PatchForensics~\cite{patchfor}, Co-occurrence~\cite{co-occurence}, Freq-Spec~\cite{freq-spec}, GramNet~\cite{gramnet}, UniFD~\cite{univfd}, LGrad~\cite{lgrad}, NPR~\cite{npr}, FreqNet~\cite{frenet}, FatFormer~\cite{fatformer}, CLIPMoLE~\cite{clipmode}, DRCT~\cite{drct}, C2P~\cite{c2p}, and MPFT~\cite{mpft}.


\subsection{Main Results}
\begin{table*}[t]
\centering
\caption{Our method outperforms existing methods on five cross-domain datasets. ``--'' indicates that the results are not reported in the original paper, and † indicates that the results are obtained by using the official pre-trained model or reproduction, and some results are cited from~\cite{effort, x2dfd}. Top results are highlighted in \textbf{bold}.}
\label{sec6-tab2}
\setlength{\tabcolsep}{6pt}
\renewcommand{\arraystretch}{0.9}
\resizebox{\linewidth}{!}{%
\begin{tabular}{c|cccccc|c|cccccc}
\toprule
\multirow{2}{*}{\textbf{Methods}} 
& \multicolumn{6}{c|}{\textbf{Video-level AUC}} 
& \multirow{2}{*}{\textbf{Methods}} 
& \multicolumn{6}{c}{\textbf{Frame-level AUC}} \\
\cmidrule(lr){2-7} \cmidrule(lr){9-14}
& CDv2 & CDv1 & DFDC & DFDCP & DFD & \textbf{Avg.}
&
& CDv2 & CDv1 & DFDC & DFDCP & DFD & \textbf{Avg.} \\
\midrule \midrule
F3Net (ECCV'20) & 78.9 & -- & 71.8 & 74.9 & 84.4 & 77.5 & Efficient-b4† (ICML'19) & 74.9 & 70.5 & 69.6 & 72.8 & 81.5 & 73.9 \\
CLIP† (ICML'21) & 88.4 & 89.6 & 83.3 & 86.3 & 92.6 & 88.0 & CLIP† (ICML'21) & 83.2 & 86.1 & 80.7 & 82.4 & 87.7 & 84.0 \\
LaV (MM'23) & 92.2 & 84.7 & 82.0 & 89.7 & 96.8 & 89.1 & Face X-Ray (CVPR'20)  & 67.9  & 70.9  & 63.3 & 69.4 & 76.7 & 69.6 \\
IID (CVPR'23) & 83.8 & -- & 70.0 & 68.9 & 93.9 & 79.2 & RECCE (CVPR'22) & 73.2 & -- & 71.3 & 74.2 & 74.2 & 73.2 \\
CFM (TIFS'23) & 89.7 & -- & 70.6 & 80.2 & 95.2 & 83.9 & UCF† (ICCV'23) & 82.4 & 77.9 & 80.5 & 75.9 & 87.8 & 80.9 \\
SBIs (CVPR'22) & 90.6 & -- & 75.2 & 87.7 & 88.2 & 85.4 & ICT (CVPR'23) & 85.7 & 81.4 & -- & -- & 84.1 & 83.7 \\
ProDet† (NeurIPS'24) & 92.6 & 95.7 & 70.7 & 82.8 & 95.4 & 87.4 & ProDet† (NeurIPS'24) & 84.5 & 90.9 & 72.4 & 81.2 & -- & 82.3 \\
CDFA† (ECCV'24) & 93.8 & 92.1 & 83.0 & 88.1 & 95.4 & 90.5 & DFFake (NeurIPS'24) & 80.5 & -- & -- & 81.0 & 90.4 & 84.0 \\
FakeRadar (ICCV'25) & 91.7 & -- & 84.1 & 88.5 & 96.2 & 90.1 & ExposDe† (AAAI'24) & 86.4 & 81.8 & 72.1 & 85.1 & 89.8 & 83.0 \\
VbSaT (CVPR'25) & 94.7 & -- & 84.3 & 90.9 & 96.5 & 91.6 & DAID (NeurIPS'25) & 84.4 & -- & 66.9 & -- & 91.2 & 80.8 \\
DeepShield (ICCV'25) & 92.2 & -- & 82.8 & \textbf{93.2} & 96.1 & 91.1 & UDD (AAAI'25) & 86.9 & -- & 75.8 & 85.6 & 91.0 & 84.8 \\
FCG (CVPR'25) & 95.0 & -- & 81.8 & -- & -- & 88.4 & FSgD (NeurIPS'25) & 86.7 & 90.1 & -- & 81.8 & 82.1 & 85.2 \\
$\chi^2\text{-DFD}$ (NeurIPS'25) & 95.5 & -- & 85.3 & 91.2 & 95.7 & 91.9 & LSDA (CVPR'24) & 83.0 & 86.7 & 73.6 & 81.5 & 88.0 & 82.6 \\
Effort† (ICML'25) & 95.6 & 91.8 & 84.3 & 90.9 & 96.5 & 91.8 & FreqDebias (CVPR'25) & 83.6 & 87.5 & 74.1 & 82.4 & 86.8 & 82.9 \\
\midrule
\rowcolor{iceblue}
GSD-CLIP & \textbf{96.4} & \textbf{97.0} & \textbf{88.3} & \textbf{93.2} & \textbf{97.7} & \textbf{94.5} 
& GSD-CLIP & \textbf{90.6} & \textbf{91.3} & \textbf{85.7} & \textbf{90.1} & \textbf{94.5} & \textbf{90.4} \\
\bottomrule
\end{tabular}%
}
\end{table*}

\begin{table*}[t]
\centering
\caption{Our method achieves superior results on unseen manipulation methods. Results are reported on six representative face-swapping methods selected from DF40, where all samples are generated based on the FF++ domain.}
\label{sec6-tab3}
\setlength{\tabcolsep}{2pt}
\renewcommand{\arraystretch}{0.95}
\resizebox{\linewidth}{!}{%
\begin{tabular}{c|ccccccc|c|ccccccc}
\toprule
\multirow{2}{*}{\textbf{Methods}} 
& \multicolumn{7}{c|}{\textbf{Video-level AUC}} 
& \multirow{2}{*}{\textbf{Methods}} 
& \multicolumn{7}{c}{\textbf{Frame-level AUC}} \\
\cmidrule(lr){2-8} \cmidrule(lr){10-16}
& UniFace & e4s & FaceDancer & FSGAN & InSwap & SimSwap & \textbf{Avg.}
& 
& UniFace & e4s & FaceDancer & FSGAN & InSwap & SimSwap & \textbf{Avg.} \\ 
\midrule \midrule
CLIP (ICML'21) & 91.0 & 91.6 & 89.8 & 96.0 & 88.5 & 80.3 & 89.5 & CLIP (ICML'21) & 86.5 & 87.8 & 85.6 & 92.0 & 83.9 & 76.3 & 85.4 \\
RECCE (CVPR'22) & 89.8 & 68.3 & 84.8 & 94.9 & 84.8 & 76.8 & 83.2 & RECCE (CVPR'22) & 84.2 & 65.2 & 78.3 & 88.4 & 79.5 & 73.0 & 78.1 \\
SBI (CVPR'22) & 72.4 & 75.0 & 59.4 & 80.3 & 71.2 & 70.1 & 71.4 & IID (CVPR'23) & 79.5 & 71.0 & 79.0 & 86.4 & 74.4 & 64.0 & 75.7 \\
UCF (ICCV'23) & 83.1 & 73.1 & 86.2 & 93.7 & 80.9 & 64.7 & 80.3 & SBI (CVPR'22) & 64.4 & 69.0 & 44.7 & 87.9 & 63.3 & 56.8 & 64.4 \\
IID (CVPR'23) & 83.9 & 76.6 & 84.4 & 92.7 & 78.9 & 64.4 & 80.2 & UCF (ICCV'23) & 78.7 & 69.2 & 80.0 & 88.1 & 76.8 & 64.9 & 76.3 \\
LSDA (CVPR'24) & 87.2 & 69.4 & 72.1 & 93.9 & 85.5 & 69.3 & 79.6 & LSDA (CVPR'24) & 85.4 & 68.4 & 75.9 & 83.2 & 81.0 & 72.7 & 77.8 \\
ProDet (NeurIPS'24) & 90.8 & 77.1 & 74.7 & 92.8 & 83.7 & 84.4 & 83.9 & CDFA (ECCV'24) & 76.5 & 67.4 & 75.4 & 84.8 & 72.0 & 76.1 & 75.4 \\
CDFA (ECCV'24) & 76.2 & 63.1 & 80.3 & 94.2 & 77.2 & 75.7 & 77.8 & ProDet (NeurIPS'24) & 84.5 & 71.0 & 73.6 & 86.5 & 78.8 & 77.8 & 78.7 \\
VbSaT (CVPR'25) & 96.0 & 98.0 & 91.6 & 96.4 & 93.7 & 93.1 & 94.8 & FSgD (NeurIPS'25) & 91.8 & 87.5 & 83.0 & 86.3 & 87.4 & 91.0 & 87.8 \\
Effort (ICML'25) & 96.2 & 98.3 & 92.6 & 95.7 & 93.6 & 92.6 & 94.8 & $\chi^2\text{-DFD}$ (NeurIPS'25) & 85.2 & 91.2 & 83.8 & 89.9 & 78.4 & 84.9 & 85.6 \\ 
\midrule
\rowcolor{iceblue}
GSD-CLIP & \textbf{98.6} & \textbf{99.5} & \textbf{97.5} & \textbf{98.8} & \textbf{96.5} & \textbf{96.0} & \textbf{97.8} 
& GSD-CLIP & \textbf{96.3} & \textbf{98.0} & \textbf{93.8} & \textbf{96.1} & \textbf{92.1} & \textbf{91.1} & \textbf{94.6} \\
\bottomrule
\end{tabular}%
}
\end{table*}

\subsubsection{Face Forgery Detection} 
\textbf{Cross-Dataset Generalization Results.}
To evaluate robustness to unseen domains, we train all models exclusively on FF++ and directly evaluate them on five cross-domain benchmarks. As reported in Table~\ref{sec6-tab2}, the proposed GSD achieves the best overall performance at both the video and frame levels. In terms of video-level AUC, our method attains an average of \textbf{94.5\%}, surpassing the strongest competing method, $\chi^2$-DFD~\cite{x2dfd}, which achieves 91.9\%, by \textbf{2.6\%}. More specifically, GSD achieves the best performance on Celeb-DF-v2, Celeb-DF-v1, DFDC, and DFD, and the improvement on DFDC is particularly notable, where our method improves the best result from 85.3\% to \textbf{88.3\%}, indicating stronger robustness under complex real-world distribution shifts. A similar trend is observed at the frame level. GSD achieves an average frame-level AUC of \textbf{90.4\%}, outperforming the strongest competing method, FSgD~\cite{fsgd}, which obtains 85.2\%, by \textbf{5.2\%}. These results demonstrate that the proposed geometric semantic decoupling strategy improves cross-domain generalization by suppressing semantically dominant yet forensically uninformative components in the learned representation.

\noindent \textbf{Cross-Manipulation Generalization Results. }We further evaluate the generalization capability of GSD on unseen manipulation methods using DF40. As shown in Table~\ref{sec6-tab3}, our method achieves superior performance across all six representative face-swapping techniques. Specifically, GSD attains an average video-level AUC of \textbf{97.8\%}, exceeding the strongest prior methods, Effort~\cite{effort} and VbSaT~\cite{vbsat}, both of which achieve 94.8\%, by \textbf{3.0\%}. At the frame level, GSD further reaches an average AUC of \textbf{94.6\%}, surpassing the previous best result of FSgD~\cite{fsgd} at 87.8\% by a large margin of \textbf{6.8\%}. In addition to the overall gains, our method achieves the best result on every individual manipulation type, with particularly clear improvements on challenging methods such as FaceDancer and SimSwap. These results indicate that the proposed geometric semantic decoupling strategy enables the model to focus on more intrinsic and manipulation-invariant forensic cues, thereby substantially improving robustness to unseen forgery pipelines.

\subsubsection{Synthetic Image Detection}
\begin{table*}[t]
\centering
\renewcommand{\arraystretch}{0.95}
\caption{ACC (\%) performance on the UniversalFakeDetect dataset. Following~\cite{univfd}, 
ProGAN is used for training, and selected results are cited from~\cite{effort}. 
Our method achieves the best average performance and demonstrates competitive or superior results across most generative categories.}
\label{sec4-tabn1}
\resizebox{\textwidth}{!}{%
\begin{tabular}{c|cccccc|c|cc|cc|c|ccc|ccc|c|c}
\toprule
\multirow{2}{*}{Methods} & \multicolumn{6}{c|}{GAN} & \multirow{2}{*}{\shortstack{Deep\\fakes}} & \multicolumn{2}{c|}{Low level} & \multicolumn{2}{c|}{Perceptual loss} & \multirow{2}{*}{Guided} & \multicolumn{3}{c|}{LDM} & \multicolumn{3}{c|}{Glide} & \multirow{2}{*}{Dalle} & \multirow{2}{*}{\textbf{Avg.}} \\
\cmidrule{2-7} \cmidrule{9-12} \cmidrule{14-19}
 & \shortstack{Pro-\\GAN} & \shortstack{Cycle-\\GAN} & \shortstack{Big-\\GAN} & \shortstack{Style-\\GAN} & \shortstack{Gau-\\GAN} & \shortstack{Star-\\GAN} & & SITD & SAN & CRN & IMLE & & \shortstack{200\\steps} & \shortstack{200\\w/cfg} & \shortstack{100\\steps} & \shortstack{100\\27} & \shortstack{50\\27} & \shortstack{100\\10} & & \\
\midrule
CNN-Spot (CVPR'20) & \textbf{100.0} & 85.2 & 70.2 & 85.7 & 79.0 & 91.7 & 53.5 & 66.7 & 48.7 & 86.3 & 86.3 & 60.1 & 54.0 & 55.0 & 54.1 & 60.8 & 63.8 & 65.7 & 55.6 & 69.6 \\
Patchfor (ECCV'20) & 75.0 & 69.0 & 68.5 & 79.2 & 64.2 & 63.9 & 75.5 & 75.1 & 75.3 & 72.3 & 55.3 & 67.4 & 76.5 & 76.1 & 75.8 & 74.8 & 73.3 & 68.5 & 67.9 & 71.2 \\
Co-occurence (ArXiv'20) & 97.7 & 63.1 & 53.8 & 92.5 & 51.1 & 54.7 & 57.1 & 63.1 & 55.9 & 65.7 & 65.8 & 60.5 & 70.7 & 70.5 & 71.0 & 70.2 & 69.6 & 69.9 & 67.5 & 66.9 \\
Freq-spec (WIFS'20) & 49.9 & \textbf{99.9} & 50.5 & 49.9 & 50.3 & 99.7 & 50.1 & 50.0 & 48.0 & 50.6 & 50.1 & 50.9 & 50.4 & 50.4 & 50.3 & 51.7 & 51.4 & 50.4 & 50.0 & 55.5 \\
F3Net (ECCV'20) & 99.4 & 76.4 & 65.3 & 92.6 & 58.1 & \textbf{100.0} & 63.5 & 54.2 & 47.3 & 51.5 & 51.5 & 69.2 & 68.2 & 75.3 & 68.8 & 81.7 & 83.2 & 83.0 & 66.3 & 71.3 \\
CLIP$\dagger$ (ICML'21)  & \textbf{100.0} & 97.0 & 98.1 & 82.3 & \textbf{100.0} & 87.6 & 53.0 & 59.1 & 50.0 & 73.1 & 90.3 & 48.0 & 81.3 & 61.6 & 80.4 & 57.2 & 59.8 & 57.9 & 81.4 & 74.6 \\
UniFD (CVPR'23) & \textbf{100.0} & 98.5 & 94.5 & 82.0 & 99.5 & 97.0 & 66.6 & 63.0 & 57.5 & 59.5 & 72.0 & 70.0 & 94.2 & 73.8 & 94.4 & 79.1 & 79.8 & 78.1 & 86.8 & 81.4 \\
LGrad (CVPR'23) & 99.8 & 85.4 & 82.9 & 94.8 & 72.5 & 99.6 & 58.0 & 62.5 & 50.0 & 50.7 & 50.8 & 77.5 & 94.2 & 95.8 & 94.8 & 87.4 & 90.7 & 89.5 & 88.3 & 80.3 \\
FreqNet (AAAI'24) & 97.9 & 95.8 & 90.5 & 97.5 & 90.2 & 93.4 & 97.4 & 88.9 & 59.0 & 71.9 & 67.3 & 86.7 & 84.5 & \textbf{99.6} & 65.6 & 85.7 & 97.4 & 88.2 & 59.1 & 85.1 \\
NPR (CVPR'24) & 99.8 & 95.0 & 87.5 & 96.2 & 86.6 & 99.8 & 76.9 & 66.9 & \textbf{98.6} & 50.0 & 50.0 & 84.5 & 97.7 & 98.0 & 98.2 & 96.2 & 97.2 & \textbf{97.3} & 87.2 & 87.6 \\
FatFormer (CVPR'24) & 99.9 & 99.3 & 99.5 & 97.2 & 99.4 & 99.8 & 93.2 & 81.1 & 68.0 & 69.5 & 69.5 & 76.0 & 98.6 & 94.9 & 98.7 & 94.3 & 94.7 & 94.2 & 98.8 & 90.9 \\
C2P-CLIP (AAAI'25) & \textbf{100.0} & 97.3 & 99.1 & 96.4 & 99.2 & 99.6 & 93.8 & \textbf{95.6} & 64.4 & 93.3 & 93.3 & 69.1 & 99.3 & 97.3 & 99.3 & 95.3 & 95.3 & 96.1 & 98.6 & 93.8 \\
MPFT (Arxiv'26) & 99.9 & 99.2 & 97.3 & 94.9 & 97.6 & 99.7 & \textbf{97.6} & 94.7 & 90.9 & 87.7 & 84.8 & 82.0 & 99.3 & 97.2 & 99.5 & 95.5 & 96.3 & 96.8 & 98.4 & 94.6 \\
\midrule
\rowcolor{iceblue}
GSD-CLIP & \textbf{100.0} & 99.4 & \textbf{99.7} & \textbf{98.3} & 99.8 & \textbf{100.0} & 91.4 & 90.5 & 91.1 & \textbf{98.3} & \textbf{98.3} & \textbf{91.9} & \textbf{99.6} & 98.2 & \textbf{99.7} & \textbf{97.8} & \textbf{98.4} & 97.2 & \textbf{99.5} & \textbf{97.3} \\
\bottomrule
\end{tabular}%
}
\end{table*}

\begin{table*}[h]
\caption{ACC (\%) performance on the GenImage dataset. Our method achieves the best average ACC while maintaining balanced performance across diverse generative categories.}
\centering
\tiny
\label{sec6-tab8}
\setlength{\tabcolsep}{6pt}
\renewcommand{\arraystretch}{0.95}
\resizebox{\textwidth}{!}{%
\begin{tabular}{c|cccccccc|c}
\hline
Methods & Midjourney & SDv1.4 & SDv1.5 & ADM & GLIDE & Wukong & VQDM & BigGAN & \textbf{Avg.} \\
\hline
ResNet-50 (CVPR'16) & 54.9 & 99.9 & 99.7 & 53.5 & 61.9 & 98.2 & 56.6 & 52.0 & 72.1 \\
DeiT-S (ICML'21)    & 55.6 & 99.9 & 99.8 & 49.8 & 58.1 & 98.9 & 56.9 & 53.5 & 71.6 \\
Swin-T (ICCV'21)    & 62.1 & 99.9 & 99.8 & 49.8 & 67.6 & 99.1 & 62.3 & 57.6 & 74.8 \\
CNNSpot (CVPR'20)   & 52.8 & 96.3 & 95.9 & 50.1 & 39.8 & 78.6 & 53.4 & 46.8 & 64.2 \\
Freq-spec (WIFS'20) & 52.0 & 99.4 & 99.2 & 49.7 & 49.8 & 94.8 & 55.6 & 49.8 & 68.8 \\
F3Net (ECCV'20)     & 50.1 & 99.9 & \textbf{99.9} & 49.9 & 50.0 & \textbf{99.9} & 49.9 & 49.9 & 68.7 \\
GramNet (CVPR'20)   & 54.2 & 99.2 & 99.1 & 50.3 & 54.6 & 98.9 & 50.8 & 51.7 & 69.9 \\
UnivFD (CVPR'23)    & \textbf{91.5} & 96.4 & 96.1 & 58.1 & 73.4 & 94.5 & 67.8 & 57.7 & 79.4 \\
NPR (CVPR'24)       & 81.0 & 98.2 & 97.9 & 76.9 & 89.8 & 96.9 & 84.1 & 84.2 & 88.6 \\
FreqNet (AAAI'24)   & 89.6 & 98.8 & 98.6 & 66.8 & 86.5 & 97.3 & 75.8 & 81.4 & 86.9 \\
DRCT (ICML'24)      & \textbf{91.5} & 95.0 & 94.4 & 79.4 & 89.2 & 94.7 & 90.0 & 81.7 & 89.5 \\
Effort (ICML'25)    & 82.4 & 99.8 & 99.8 & 78.7 & \textbf{93.3} & 97.4 & 91.7 & 77.6 & 90.1 \\
\hline
\rowcolor{iceblue}
GSD-CLIP & 84.7 & \textbf{100.0} & \textbf{99.9} & \textbf{87.2} & 92.1 & 99.8 & \textbf{95.3} & \textbf{91.0} & \textbf{93.8} \\
\hline
\end{tabular}%
}
\end{table*}

\textbf{UniversalFakeDetect Generalization Results.} Following the protocol of~\cite{univfd,cnnspot,effort}, we train the detector only on ProGAN-generated images and their real counterparts, and evaluate its cross-generator generalization on the UniversalFakeDetect benchmark. As shown in Table~\ref{sec4-tabn1}, our method achieves the best average accuracy of \textbf{97.3\%}, outperforming the strongest prior method, MPFT~\cite{mpft}, by \textbf{2.7} percentage points (97.3\% vs.\ 94.6\%), and surpassing C2P-CLIP~\cite{c2p} by \textbf{3.5} percentage points (97.3\% vs.\ 93.8\%). A closer inspection shows that the improvement does not come from a single favorable subset, but from broadly strong performance across diverse generator families. In the GAN-based categories, our method remains near-saturated on most subsets, while also showing clear advantages on more challenging unseen regimes such as perceptual-loss-based generation, guided diffusion, LDM, Glide, and DALL-E. We also observe that our method is not uniformly superior on every individual subset. For example, MPFT obtains higher accuracy on Deepfakes and SITD, and some earlier methods achieve very strong results on specific categories such as CycleGAN. Nevertheless, our method achieves the highest overall average and delivers more balanced performance across all generators. This suggests that explicitly suppressing semantically dominant directions inherited from large-scale pre-training helps reduce the detector's reliance on category-level semantics, thereby encouraging the representation to capture more intrinsic forensic irregularities that are transferable across heterogeneous generative models.

\noindent \textbf{GenImage Generalization Results.} For the GenImage dataset, we follow the standard setting by using SDv1.4 for training and evaluating generalization across different generative models. As shown in Table~\ref{sec6-tab8}, our method achieves the best average accuracy of \textbf{93.8\%}, outperforming the strongest prior method, Effort~\cite{effort}, by \textbf{3.7} percentage points (93.8\% vs.\ 90.1\%). This improvement indicates that GSD-CLIP provides stronger overall cross-generator generalization rather than only improving performance on the source domain. A more detailed comparison shows that our method achieves the best results on several challenging target generators, including ADM, VQDM, and BigGAN, and also reaches near-saturated performance on SDv1.4, SDv1.5, and Wukong. These results suggest that the proposed GSD helps the detector learn more transferable forensic representations from a single source generator, reducing its reliance on generator-specific semantic patterns and improving robustness under cross-generator distribution shifts. 


\subsection{Comparison of Training and Inference Protocol}
\label{sec6-3}
In the proposed GSD framework, the choice of inference protocol must account for realistic deployment constraints as well as the trade-off between training consistency, inference efficiency, and semantic estimation fidelity. While batch-level semantic decoupling is naturally aligned with the training formulation and large-scale benchmark evaluation, real-world applications often require per-sample processing, where batch statistics may be unavailable. Meanwhile, reference-based strategies offer a practical alternative for reducing the cost of per-image semantic decomposition when auxiliary samples can be prepared beforehand. Therefore, it is necessary to discuss different training and inference protocols explicitly, not only to ensure fair evaluation, but also to clarify how the proposed method can be adapted to different practical scenarios. In the following, we consider two training strategies, namely \emph{Single-SVD} training and \emph{Batch-SVD} training, together with three inference protocols for GSD: Batch Inference (\emph{B-Infer}), Per-sample Inference (\emph{P-Infer}), and Reference-based Inference (\emph{R-Infer}).

\begin{table}[t]
\centering
\caption{Comparison of different training and inference protocols for face forgery detection. We report ACC and frame-level AUC (\%) on 11 test datasets.}
\setlength{\tabcolsep}{2pt}
\label{sec6-tab9}
\renewcommand{\arraystretch}{0.9}
\small
\resizebox{\columnwidth}{!}{%
\begin{tabular}{l c c c c c c c}
\toprule
\textbf{Infer} & \textbf{Ref. Dataset} & \textbf{Metric} & \textbf{CD2} & \textbf{DFDC} & \textbf{e4s} & \textbf{FaceDancer} & \textbf{Avg.all} \\
\midrule

\rowcolor{SoftCyan}
\multicolumn{8}{c}{\textbf{Single-SVD Training for Face Forgery}} \\
\midrule
\textit{P-Infer} & -- & ACC & \underline{83.0} & \textbf{77.2} & \textbf{93.7} & \textbf{87.3} & \textbf{86.0} \\
\textit{P-Infer} & -- & AUC & \underline{90.4} & \textbf{86.2} & \textbf{98.1} & \textbf{94.1} & \textbf{92.8} \\
\midrule

\rowcolor{SoftPeach}
\multicolumn{8}{c}{\textbf{Batch-SVD Training for Face Forgery}} \\
\midrule
\textit{P-Infer} & -- & ACC & 82.7 & 76.4 & 93.1 & \underline{87.0} & 85.6 \\
\textit{P-Infer} & -- & AUC & 90.4 & 85.5 & 97.9 & \underline{93.9} & 92.4 \\
\addlinespace[1pt]

\textit{R-Infer} & FF++ & ACC & 82.6 & 76.9 & 93.3 & 86.7 & 85.7 \\
\textit{R-Infer} & FF++ & AUC & 90.3 & 85.5 & 97.9 & 93.7 & 92.3 \\
\addlinespace[1pt]

\textit{R-Infer} & CDv2  & ACC & 82.6 & 76.4  & 93.2  & 86.7  & 85.7   \\
\textit{R-Infer} & CDv2 & AUC & 90.4 & 85.5  & 97.9  & 93.7  & 92.3   \\
\addlinespace[1pt]

\textit{B-Infer} & -- & ACC & \textbf{83.1} & \underline{77.0} & \underline{93.5} & 86.9 & \underline{85.9} \\
\textit{B-Infer} & -- & AUC & \textbf{90.6} & \underline{85.8} & \underline{98.0} & 93.8 & \underline{92.7}\\
\bottomrule
\end{tabular}%
}
\end{table}

\begin{table}[t]
\centering
\caption{Comparison of different inference protocols for the natural-image setting. ACC (\%) is reported on category-specific images generated by CycleGAN and ProGAN.}
\label{sec6-tab10}
\setlength{\tabcolsep}{3.5pt}
\renewcommand{\arraystretch}{0.9}
\resizebox{0.98\linewidth}{!}{
\begin{tabular}{c c cccc c}
\toprule
\multirow{2}{*}{\textbf{Infer}} & \multirow{2}{*}{\textbf{Metric}} & \multicolumn{5}{c}{\cellcolor{SoftPeach}\textbf{CycleGAN}}  \\
\cmidrule(lr){3-7}
& & \textbf{Apple} & \textbf{horse} & \textbf{Orange} & \textbf{Summer} & \textbf{Avg.all} \\
\midrule
\textit{R-Infer}  & ACC & 99.6 & \textbf{99.8} & 99.4 & 99.3 & 99.4 \\
\textit{B-Infer}  & ACC & 99.6 & 99.7 & 99.4 & 99.3 & 99.4 \\
\midrule
\multirow{2}{*}{\textbf{Infer}} & \multirow{2}{*}{\textbf{Metric}} & \multicolumn{5}{c}{\cellcolor{SoftCyan}\textbf{ProGAN}} \\
\cmidrule(lr){3-7}
& & \textbf{Airplane} & \textbf{Bicycle} & \textbf{Car} & \textbf{Dog} & \textbf{Avg.all} \\
\midrule
\textit{R-Infer} & ACC & 100.0 & 100.0 & 100.0 & 100.0 & 100.0 \\
\textit{B-Infer} & ACC & 100.0 & 100.0 & 100.0 & 100.0 & 100.0 \\
\bottomrule
\end{tabular}}
\end{table}

\subsubsection{Different Strategies for Face Image} 
We first investigate the impact of different training and inference protocols in the face forgery setting. Specifically, the detector is trained on FF++ and evaluated over 11 cross-domain test sets. For Reference-based Inference, we consider two strategies for constructing the reference set. Firstly, we randomly sample 512 images from FF++, corresponding to a setting in which the reference set is semantically aligned with the training distribution. Secondly, we instead sample 512 face images from CDv2, whose distribution differs from that of the training set, in order to better simulate realistic deployment scenarios where the test data are not distributionally matched to the training data. The results are reported in Table~\ref{sec6-tab9}.

Several observations can be drawn from Table~\ref{sec6-tab9}. First, for models trained with Batch-SVD, both Per-sample Inference and Reference-based Inference achieve performance very close to that of Batch Inference. Notably, the effectiveness of Reference-based Inference does not rely on the reference set being drawn from the training data itself; it remains competitive even when the reference images are sampled from an external dataset. Second, compared with Batch-SVD training, Single-SVD training yields only marginal performance gains. This suggests that GSD does not require computationally intensive per-sample semantic subspace estimation during optimization. Instead, estimating the semantic subspace at the mini-batch level is already sufficient to capture the dominant anchor-induced semantic structure and to learn a robust forgery-oriented representation. These findings are consistent with the hypothesis discussed in Section~\ref{sec:inference}, and at test time, even if the semantic subspace is estimated in a different manner, \emph{e.g.}, through per-image decomposition or a pre-defined reference set, the resulting decoupled features remain well aligned with the forgery-oriented manifold during training. From a practical perspective, these results further underscore the deployability of the proposed framework: it can be trained efficiently with Batch-SVD, while supporting multiple inference protocols that remain effective under realistic application constraints, including scenarios where batch statistics are unavailable or external reference sets must be used.

\subsubsection{Different Inference Protocols for Natural Image}
We further investigate the effect of different inference strategies in the natural-image setting. Specifically, we select the CycleGAN and ProGAN subsets from UniversalFakeDetect, as they provide explicit semantic categories for each generated image, making it possible to construct category-aligned reference sets for Reference-based Inference. We then sample semantically matched natural images from ImageNet-12K and Pascal Visual Object Classes (VOC) 2012 as the reference sets, thereby simulating a realistic deployment scenario in which auxiliary reference images are available but are not drawn from the training distribution itself. The corresponding results are reported in Table~\ref{sec6-tab10}. The performance gap between Per-sample Inference and Batch Inference is negligible on both CycleGAN and ProGAN, indicating that the proposed GSD remains highly stable even when batch-level statistics are unavailable at test time. This observation is consistent with our analysis in the face-image setting, suggesting that the effectiveness of GSD does not arise from strictly reproducing the exact semantic-subspace estimation protocol during inference. Rather, once the forgery-oriented representation manifold has been established during training, different inference strategies can still produce highly consistent decoupled features under the same anchor-induced semantic geometry.


\subsection{Ablation Study}
\label{sec6-4}
\subsubsection{The Impact of Different $\lambda$}
We investigate the effect of the subtraction coefficient $\lambda$ on detection performance. To this end, we replace the adaptive subtraction coefficient with a set of fixed constants and evaluate the resulting models under the same protocol. We train our method on FF++, and the results are reported in Table~\ref{sec6-tab11}. Two observations can be made. First, increasing $\lambda$ generally leads to improved performance, indicating that stronger suppression of semantically dominant components is beneficial for learning more transferable forensic representations. In particular, ACC improves steadily as $\lambda$ increases from 0.05 to 0.6 on most evaluated datasets. However, once $\lambda$ becomes sufficiently large, the performance gain begins to saturate, and the difference between $\lambda=0.6$ and $\lambda=0.8$ is marginal. Second, the adaptive formulation consistently achieves better overall performance than any fixed setting. Since different inputs exhibit different levels of semantic dominance, a sample-adaptive coefficient enables the model to modulate the subtraction strength more appropriately, thereby better balancing the preservation of discriminative forensic cues against the suppression of semantically interfering information. For example, on the InSwap dataset, the adaptive strategy improves the ACC from 84.6\% under $\lambda=0.8$ to 85.1\%, yielding a gain of 0.5 percentage points. A similar trend can be observed in the overall average, where the adaptive setting achieves the best ACC among all compared strategies.

\begin{figure}[t]
    \centering
    \includegraphics[width=\linewidth]{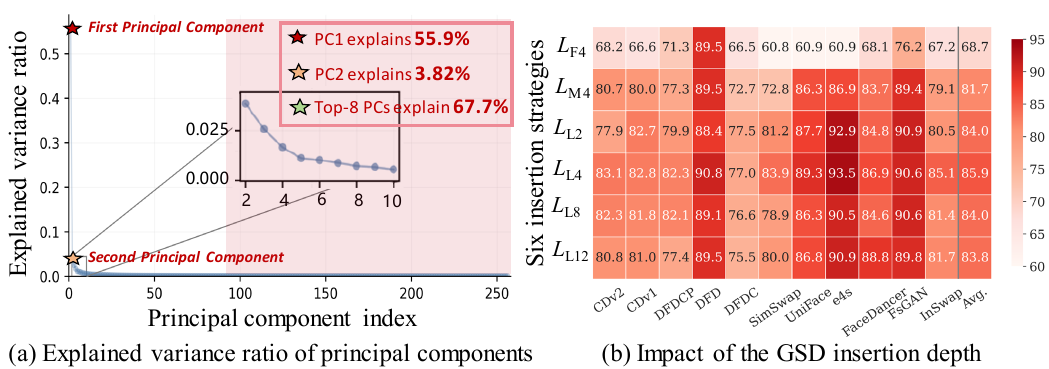}
    \caption{(a) The semantic variance is highly concentrated in the first few principal components, particularly the first and second ones, while the remaining components contribute only marginally. (b) We compare six insertion strategies and report ACC (\%) on the testing datasets.}
    \label{sec6-fig5}
\end{figure}

\begin{table}[t]
\centering
\caption{Comparison under different subtraction strategies. We report ACC (\%) on the testing datasets.}
\label{sec6-tab11}
\setlength{\tabcolsep}{3pt}
\renewcommand{\arraystretch}{0.9}
\resizebox{0.98\linewidth}{!}{
\begin{tabular}{c|cccccc}
\toprule
$\lambda$ & \textbf{CDv2} & \textbf{DFDCP} & \textbf{DFDC} & \textbf{e4s} & \textbf{InSwap} & \textbf{Avg.} \\
\midrule
$0.05$ & 81.2 & 80.9 & 76.3 & 92.6 & 84.0 & 83.0 \\
$0.1$  & 81.4 & 80.8 & 76.6 & 92.7 & 84.5 & 83.2 \\
$0.2$  & 81.7 & 81.3 & \textbf{77.0} & \textbf{93.5} & 85.0 & 83.7 \\
$0.4$  & 82.3 & 82.0 & 76.7 & 93.2 & 84.5 & 83.7 \\
$0.6$  & 82.8 & 82.2 & 76.9 & 93.1 & 84.7 & 83.9 \\
$0.8$  & 82.8 & \textbf{82.3} & 76.8 & 93.2 & 84.6 & 83.9 \\
\midrule
\rowcolor{iceblue}
Adaptive & \textbf{83.1} & \textbf{82.3} & \textbf{77.0} & \textbf{93.5} & \textbf{85.1} & \textbf{84.2} \\
\bottomrule
\end{tabular}}
\end{table}

\begin{table}[t]
\centering
\caption{Effect of the subspace dimension $k$ on the testing datasets. Frame-level AUC (\%) is reported.}
\label{sec6-tab12}
\setlength{\tabcolsep}{5pt}
\renewcommand{\arraystretch}{0.9}
\resizebox{0.98\linewidth}{!}{
\begin{tabular}{c|cccccc}
\toprule
\textbf{$k$} & \textbf{CDv2} & \textbf{DFD} & \textbf{DFDC} & \textbf{SimSwap} & \textbf{UniFace} & \textbf{Avg.} \\
\midrule
1  & 89.8 & 94.0 & \textbf{85.9} & \textbf{91.4} & 95.9 & 91.4 \\
2  & 90.2 & 94.0 & 85.5 & 91.3 & 95.9 & 91.4 \\
8 & 90.5 & 94.4 & \textbf{85.9} & 91.2 & 96.3 & \textbf{91.7} \\
16 & 90.4 & 94.4 & 85.7 & 91.0 & \textbf{96.4} & 91.6 \\
32 & 90.5 & 94.4 & 85.6 & 91.0 & \textbf{96.4} & 91.6 \\
\rowcolor{iceblue}
64  & \textbf{90.6} & \textbf{94.5} & 85.8 & 91.1 & 96.3 & \textbf{91.7} \\
\bottomrule
\end{tabular}}
\end{table}


\subsubsection{The Impact of the Estimated Subspace Dimension $k$}
We further investigate the effect of the estimated semantic subspace dimension $k$ on detection performance, as reported in Table~\ref{sec6-tab12}. Overall, even a very small subspace dimension already yields strong performance: setting $k=1$ produces competitive results across both cross-domain and cross-manipulation benchmarks, indicating that suppressing only the most dominant semantic direction is already beneficial. As $k$ increases to 8, the performance improves further, suggesting that a moderately richer semantic subspace allows GSD to more effectively remove semantically dominant components and thus enhance generalization. Increasing $k$ beyond 8 brings only marginal additional gains. Based on this trade-off, we adopt $k=64$ as the default setting in all experiments.

We further analyze why such small values of $k$, \emph{e.g.}, $k=1$ or $k=2$, are already highly effective. The reason is that the semantic feature space induced by CLIP is itself highly low-rank. As shown in Figure~\ref{sec6-fig5}(a), we randomly sample 5,000 images from FF++ and analyze the explained variance ratio of the principal components extracted from CLIP features. It can be seen that the semantic variance is concentrated overwhelmingly in the first few principal components, especially the first two, while the remaining components contribute only marginally. This observation is highly consistent with the trend in Table~\ref{sec6-tab12}: once the dominant low-dimensional semantic structure is removed, the benefit of further increasing $k$ quickly diminishes.

\subsubsection{The Impact of the GSD Module Placement}
The GSD insertion depth is critical for balancing semantic suppression and artifact preservation. Specifically, we evaluate six insertion depths: the first ($L_{\mathrm{F4}}$) and middle ($L_{\mathrm{M4}}$) four layers, as well as the last 2 ($L_{\mathrm{L2}}$), 4 ($L_{\mathrm{L4}}$), 8 ($L_{\mathrm{L8}}$), and 12 ($L_{\mathrm{L12}}$) layers. The results in Figure~\ref{sec6-fig5}(b) show that GSD should not be inserted too early. Applying it to the first four layers leads to a substantial performance drop, while inserting it into the middle four layers, or the last 8/12 layers all remains clearly inferior to inserting it into the last four layers. This is because shallow layers retain more low-level texture and artifact cues, whereas the deeper layers encode richer and more coherent semantic structure; such a phenomenon is consistent with prior findings that preserving shallow forensic details is important for manipulation detection~\cite{lav, laa}. It is also in line with Table~\ref{sec3-tab1}, which shows that the deeper layers of the fine-tuned CLIP encoder exhibit highly consistent semantic structure, and this semantic consistency is precisely one of the major causes of semantic fallback and degraded cross-domain generalization. At the same time, the insertion depth should not be excessively late. For example, inserting GSD only into the last two layers performs noticeably worse than inserting it into the last four layers, suggesting that semantic decoupling benefits from a moderate depth. Taken together, these results indicate that inserting GSD into the last four layers provides the best trade-off, and we therefore adopt this configuration as the default setting.

\subsubsection{Analysis of initial settings for parameters $a$ and $b$}
\begin{figure}[t]
    \centering
    \includegraphics[width=0.98\linewidth]{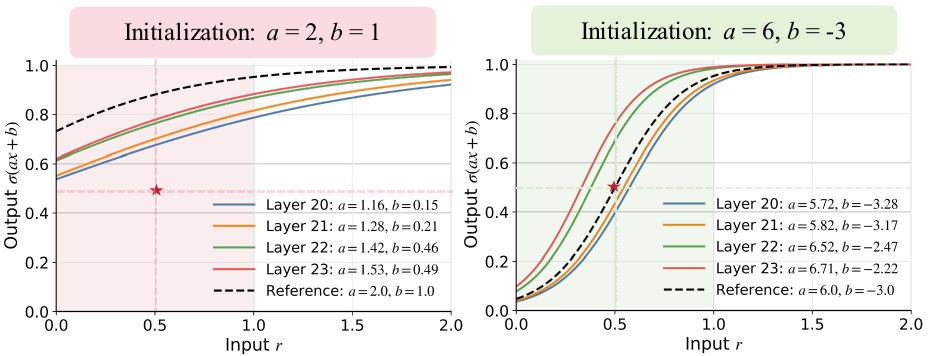}
    \vspace{3pt}
    \setlength{\tabcolsep}{4pt}
    \renewcommand{\arraystretch}{1.05}
    \resizebox{\linewidth}{!}{
    \begin{tabular}{c|ccccc|c}
    \toprule
    \textbf{$(a,b)$} & \textbf{CDv2} & \textbf{DFDC} & \textbf{DFDCP} & \textbf{SimSwap} & \textbf{FSGAN} & \textbf{Avg.} \\
    \midrule
    $(1, 0)$  & 81.9 & 76.9 & 81.1 & \textbf{84.2} & \textbf{90.9} & 83.0 \\
    $(1, 1)$  & 82.8 & 76.6 & 81.9 & \textbf{84.0} & \textbf{90.8} & 83.2 \\
    \rowcolor{iceblue}
    $(2, 1)$ & \textbf{83.1} & \textbf{77.0} & \textbf{82.1} & 83.9 & 90.8 & \textbf{83.4} \\
    $(2, -1)$ & \textbf{83.1} & 76.8 & \textbf{82.1} & 83.9 & 90.6 & 83.3 \\
    $(6, -3)$ & 82.3 & \textbf{77.0} & 81.4 & 83.8 & 90.7 & 83.0 \\
    \bottomrule
    \end{tabular}
    }
    \caption{Visualization of different initialization settings for $\lambda$ and the corresponding results (ACC, \%) on the test datasets.}
    \label{sec6-fig7}
\end{figure}
We further analyze the effect of different initialization settings of the adaptive modulation parameters $a$ and $b$, which directly control the magnitude of subtraction through the mapping to $\lambda$. Specifically, we evaluate five representative initialization configurations, and the corresponding results are summarized in Figure~\ref{sec6-fig7}. Several key observations can be drawn. 1) Among all tested configurations, $(2,1)$ achieves the best overall performance across all benchmarks. Although different initialization settings lead to certain performance variations, the overall differences remain relatively limited, indicating that the proposed adaptive mechanism is robust to the choice of initialization. 2) A clear compensatory effect is observed in the learned modulation. When the initial subtraction strength is relatively strong (\textit{e.g.,} $(1,1)$ and $(2,1)$), the model tends to reduce the effective subtraction during training; conversely, when the initial strength is weaker (\textit{e.g.,} $(6,-3)$), the model adaptively increases it. This behavior is consistent with our analysis in Section~\ref{sec3-3}, suggesting the existence of an implicit equilibrium point for $\lambda$, where semantic redundancy is sufficiently suppressed while shared discriminative evidence is preserved. 3) The learned parameters exhibit distinct layer-wise variations, indicating that each layer learns a dedicated modulation pattern according to its position in the representation hierarchy and the corresponding degree of semantic dominance. This observation further justifies the use of layer-specific adaptive coefficients rather than a globally fixed subtraction rule. 4) A consistent trend is observed across all settings: the learned $\lambda$ values progressively increase in deeper layers. This suggests that deeper layers contain more semantically dominant components and therefore require stronger suppression. This trend is also highly consistent with our earlier analysis that deeper layers of CLIP are increasingly dominated by high-level semantic structures, where the semantic fallback phenomenon becomes most pronounced. This further suggests that the proposed adaptive modulation mechanism is not merely a parameterization choice, but a principled way to align subtraction strength with the semantic structure of deep representations.

\subsubsection{The Impact of Different Batch Sizes}
\begin{table}[t]
\centering
\caption{Effect of training and inference batch sizes on frame-level AUC (\%) over the testing datasets.}
\label{sec6-tab13}
\setlength{\tabcolsep}{4pt}
\begin{tabular*}{\columnwidth}{@{\extracolsep{\fill}}lcccccc@{}}
\toprule
\textbf{Batchsize} & \textbf{        CDv1        } & \textbf{        DFD        } & \textbf{      DFDC} & \textbf{FaceDancer} & \textbf{FSGAN      } & \textbf{    Avg.     } \\
\midrule
\rowcolor{SoftGreen}
\multicolumn{7}{c}{\textbf{Different Training Batch Sizes (Test batch size $\times 4$)}} \\
\midrule
$\mathit{B_{tr}}=16$  & 90.6 & 94.4 & 85.0 & 93.6 & 96.3 & 92.0 \\
$\mathit{B_{tr}}=32$  & 90.5 & 94.2 & 85.1 & 93.7 & 96.4 & 92.0 \\
$\mathit{B_{tr}}=64$  & 90.9 & 94.1 & 86.1 & 93.6 & 96.4 & 92.2 \\
\rowcolor{SoftCyan}
$\mathit{B_{tr}}=128$ & 91.3 & 94.5 & 85.8 & 93.8 & 96.1 & 92.3 \\
\midrule
\rowcolor{SoftPeach}
\multicolumn{7}{c}{\textbf{Different Inference Batch Sizes (Training Batch Size = 128)}} \\
\midrule
$\mathit{B_{ts}}=32$  & 91.1 & 94.4 & 85.8 & 93.8 & 96.1 & 92.2 \\
$\mathit{B_{ts}}=64$  & 91.2 & 94.5 & 85.8 & 93.9 & 96.1 & 92.3 \\
$\mathit{B_{ts}}=128$ & 91.0 & 94.5 & 85.5 & 93.9 & 96.1 & 92.2 \\
$\mathit{B_{ts}}=256$ & 91.3 & 94.5 & 85.7 & 93.8 & 96.1 & 92.3 \\
\bottomrule
\end{tabular*}
\end{table}

Considering practical deployment constraints, especially limited GPU memory in real-world scenarios, we evaluate both different training batch sizes and different inference batch sizes. The corresponding results are reported in Table~\ref{sec6-tab13}. 
As the training batch size decreases from 128 to 16, the average frame-level AUC drops slightly, from 92.3\% to 92.0\%. This suggests that larger mini-batches provide a somewhat more stable estimate of the semantic subspace during optimization, but the dependence is not strong. More importantly, even under smaller training batch sizes, the proposed method still remains highly competitive, and in comparison with the results in Tables~\ref{sec6-tab2} and~\ref{sec6-tab3}, it continues to preserve state-of-the-art performance. The influence of the inference batch size is even smaller, and under the default training batch size, varying the inference batch size from 32 to 256 results in almost no change in performance. Overall, these results demonstrate that the proposed framework is robust to batch-size variation. Such stability is practically valuable, as it ensures that the method can be deployed flexibly and effectively even in scenarios with limited computational resources or memory budgets.

\subsubsection{Compatibility with Various Vision Foundation Models}
We further evaluate the compatibility of GSD with different vision foundation models, including BLIP, SigLIP, and CLIP. As shown in Table~\ref{sec6-tab14}, GSD consistently improves performance across all backbones. For example, on BLIP, GSD improves the average AUC from 82.0 to 85.9; on SigLIP, from 84.3 to 88.9; and on CLIP, from 88.8 to 94.4. Notably, the improvements are consistently observed across all evaluation datasets, indicating that the effectiveness of GSD is not tied to a specific architecture or pre-training paradigm. Overall, these results highlight that GSD establishes a general semantic-decoupling training paradigm, which can be seamlessly integrated into a wide range of VFMs to improve their  cross-domain generalization.

\begin{table}[t]
\centering
\caption{Ablation studies (video-level AUC\%) on different vision foundation models (VFMs).}
\label{sec6-tab14}
\setlength{\tabcolsep}{3pt}
\renewcommand{\arraystretch}{0.9}
\resizebox{\linewidth}{!}{
\begin{tabular}{c|cccccc}
\toprule
\textbf{Method} & \textbf{CDv2} & \textbf{DFDCP} & \textbf{DFDC} & \textbf{InSwap} & \textbf{FaceDancer} & \textbf{Avg.} \\
\midrule
BLIP & 83.3 & 81.5 & 76.9 & 87.5 & 81.0 & 82.0 \\
+Ours  & 87.4 \textsuperscript{\textcolor{growthgreen}{+4.1}} 
& 84.9 \textsuperscript{\textcolor{growthgreen}{+3.4}} 
& 80.3 \textsuperscript{\textcolor{growthgreen}{+3.4}} 
& 91.2 \textsuperscript{\textcolor{growthgreen}{+3.7}} 
& 85.7 \textsuperscript{\textcolor{growthgreen}{+4.7}} 
& 85.9 \textsuperscript{\textcolor{growthgreen}{+3.9}} \\
\midrule
SigLIP  & 88.0 & 86.3 & 78.7 & 85.8 & 82.5 & 84.3 \\
+Ours  & 91.3 \textsuperscript{\textcolor{growthgreen}{+3.3}} 
& 89.4 \textsuperscript{\textcolor{growthgreen}{+3.1}} 
& 81.0 \textsuperscript{\textcolor{growthgreen}{+2.3}} 
& 93.1 \textsuperscript{\textcolor{growthgreen}{+7.3}} 
& 89.9 \textsuperscript{\textcolor{growthgreen}{+7.4}} 
& 88.9 \textsuperscript{\textcolor{growthgreen}{+4.7}} \\
\midrule
CLIP  & 90.7 & 88.4 & 83.1 & 90.3 & 91.5 & 88.8 \\
+Ours  & 96.4 \textsuperscript{\textcolor{growthgreen}{+5.7}} 
& 93.2 \textsuperscript{\textcolor{growthgreen}{+4.8}} 
& 88.6 \textsuperscript{\textcolor{growthgreen}{+5.5}} 
& 96.5 \textsuperscript{\textcolor{growthgreen}{+6.2}} 
& 97.5 \textsuperscript{\textcolor{growthgreen}{+6.0}} 
& 94.4 \textsuperscript{\textcolor{growthgreen}{+5.6}} \\
\bottomrule
\end{tabular}}
\end{table}

\section{Conclusion and Limitations}
In this paper, we identify a critical challenge in detecting AI-generated content: \emph{semantic fallback}, a failure mechanism in VFM-based deepfake detectors where, even after fine-tuning, the models still preserve the high-level semantic feature manifold learned during pre-training, which severely suppresses the representation of subtle forgery-specific cues that are crucial for generalizing to unseen generation datasets. To address this, we propose a simple module, Geometric Semantic Decoupling (GSD), that suppresses semantic components from learned features via a geometric projection, forcing detectors to focus on subtle artifact cues. This approach not only achieves state-of-the-art generalization across diverse deepfake datasets, but also successfully extends to general-scene synthetic images, highlighting its broad applicability. By reducing semantic bias, GSD advances the reliability of AI-generated content detection, helping society more effectively identify deepfakes and other synthetic media, and mitigating the growing risks posed by generative AI.

Despite these promising results, GSD still has several limitations. Although our method explicitly decouples semantic content from forensic representations and adaptively controls the suppression strength according to the semantic occupancy of each sample, it does not precisely distinguish between semantic components that are harmful to detection and those that may still be beneficial. As our analysis suggests, semantic information and forgery-related cues are not fully independent; instead, they can be partially entangled in the learned feature space. Therefore, while the optimization process encourages the detector to move toward a more semantics-decoupled and artifact-oriented representation, the current projection-based formulation may still struggle to identify the optimal trade-off between semantic suppression and forgery-cue preservation. Future work could further improve this balance by explicitly modeling forgery-specific evidence, for example by introducing auxiliary objectives related to frequency-domain artifacts, local blending traces, or other manipulation-relevant cues.

\ifCLASSOPTIONcompsoc
\section*{Acknowledgments}
\else
\section*{Acknowledgment}
\fi
This work was supported in part by the National Natural Science Foundation of China (No. 62372402), and in part by the Zhejiang Provincial Natural Science Foundation of China under Grant No. LQN26F020059.
%

%
\bibliography{ref}

\clearpage
\section{Appendix}

\setcounter{section}{0}
\renewcommand{\thesection}{\Roman{section}}

\setcounter{theorem}{0}
\setcounter{proposition}{0}
\setcounter{assumption}{0}
\setcounter{corollary}{0}

\section{Robustness Evaluation}
To assess the resilience of our model against real-world image distortions, we follow the rigorous evaluation protocols established in prior studies~\cite{lips,effort,lsda,df1.0}. The assessment covers five distinct degradation categories: Color Saturation, JPEG Compression, Color Contrast, Block-wise artifacts, and Gaussian Blur. For each category, we apply distortions ranging from severity level 0 to 5 to measure performance stability. All models are trained on the FaceForensics++ (c23) dataset and evaluated using video-level AUC. We benchmark our approach against leading methods such as Face X-ray~\cite{face-x-ray}, SBI~\cite{sbi}, LipForensics~\cite{lips}, and FWA~\cite{fwa}. To ensure a controlled comparison, we also reproduced F3Net~\cite{f3net}, UCF~\cite{ucf}, and CORE~\cite{core} using \textbf{the same data augmentation pipelin}e employed in our framework.

The quantitative results are presented in Figure~\ref{A1-1}. As shown in the ``Average” plot, our method consistently secures the top performance across nearly all severity levels, outperforming both the standard baselines and the reproduced counterparts sharing identical augmentation schemes. While competitors experience sharp performance declines under severe distortions, specifically in JPEG compression and Gaussian blur where high-frequency details are compromised, our model exhibits superior stability. This evidence confirms that the robustness of our method is intrinsic to the learned generalized features rather than being a byproduct of data augmentation.

\begin{figure*}[b]
    \centering
    \includegraphics[width=0.98\linewidth]{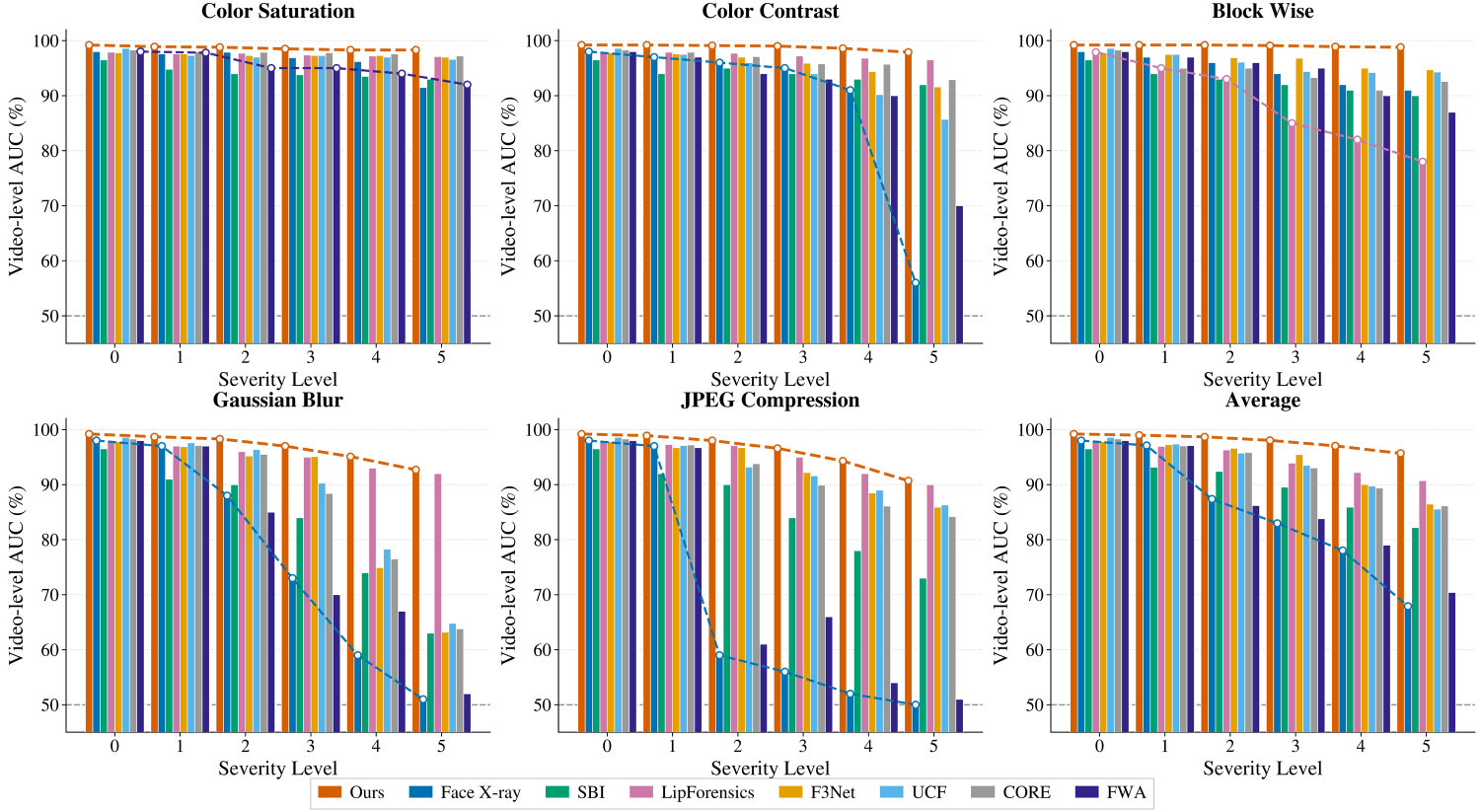}
    \caption{Robustness analysis against common image corruptions. The curves illustrate the Video-level AUC (\%) performance of our method versus existing approaches. We test across five different degradation scenarios, including Color Saturation, JPEG Compression, Color Contrast, Block-wise noise, and Gaussian Blur, plotted against increasing severity levels (0-5).}
    \label{A1-1}
\end{figure*}

\section{Feature Visualization of Different Reference Strategies}
To further examine whether different inference protocols preserve the semantic-decoupling effect of GSD, we provide 3D t-SNE visualizations for Per-sample Inference, Batch Inference, and Reference-based Inference in Figure~\ref{A1-3}. We consider both face-forgery detection and natural-image synthesis detection, covering in-domain settings such as FaceForensics++ and ProGAN, as well as cross-domain settings including CelebDF\_v2 and DFDC. For each protocol, we visualize two types of representations: the semantic components estimated by the frozen semantic anchor and the semantically decoupled forgery features produced by GSD.

As shown in the top row, the semantic components estimated by the three inference protocols occupy highly similar feature spaces. Despite differences in data domain and visual content, the semantic components consistently form structured manifolds, indicating that Per-sample, Batch, and Reference-based Inference recover comparable high-level semantic directions from the frozen CLIP anchor. This suggests that the dominant semantic subspace estimated by GSD remains stable across different inference protocols, which is consistent with our analysis in Section~\ref{sec:inference}. More importantly, the bottom row shows that after semantic decoupling, the forgery features produced by the three protocols are projected into an almost consistent forgery-oriented feature space. Compared with the original semantic components, the decoupled features are less organized by identity, object category, or scene-level semantics, and instead exhibit similar feature reorganization patterns across Per-sample, Batch, and Reference-based Inference. This indicates that Reference-based Inference preserves the essential semantic-suppression geometry of GSD, rather than inducing a substantially different representation space from the more exact Per-sample and Batch protocols.

These visualizations are consistent with the quantitative results reported in Section~\ref{sec6-3}, where the three inference protocols achieve comparable detection performance while exhibiting different computational costs. Per-sample Inference provides the most fine-grained sample-specific semantic estimation but requires per-sample SVD computation. Batch Inference improves efficiency by estimating a shared semantic basis within each mini-batch, making it suitable for dataset-level evaluation. Reference-based Inference further eliminates repeated test-time SVD by reusing a pre-constructed semantic reference space. Since all three strategies yield highly aligned semantic components and nearly unified decoupled forgery features, GSD can be deployed under different computational constraints while maintaining the learned forgery-oriented representation geometry.

\begin{figure*}[t]
    \centering
    \includegraphics[width=0.98\linewidth]{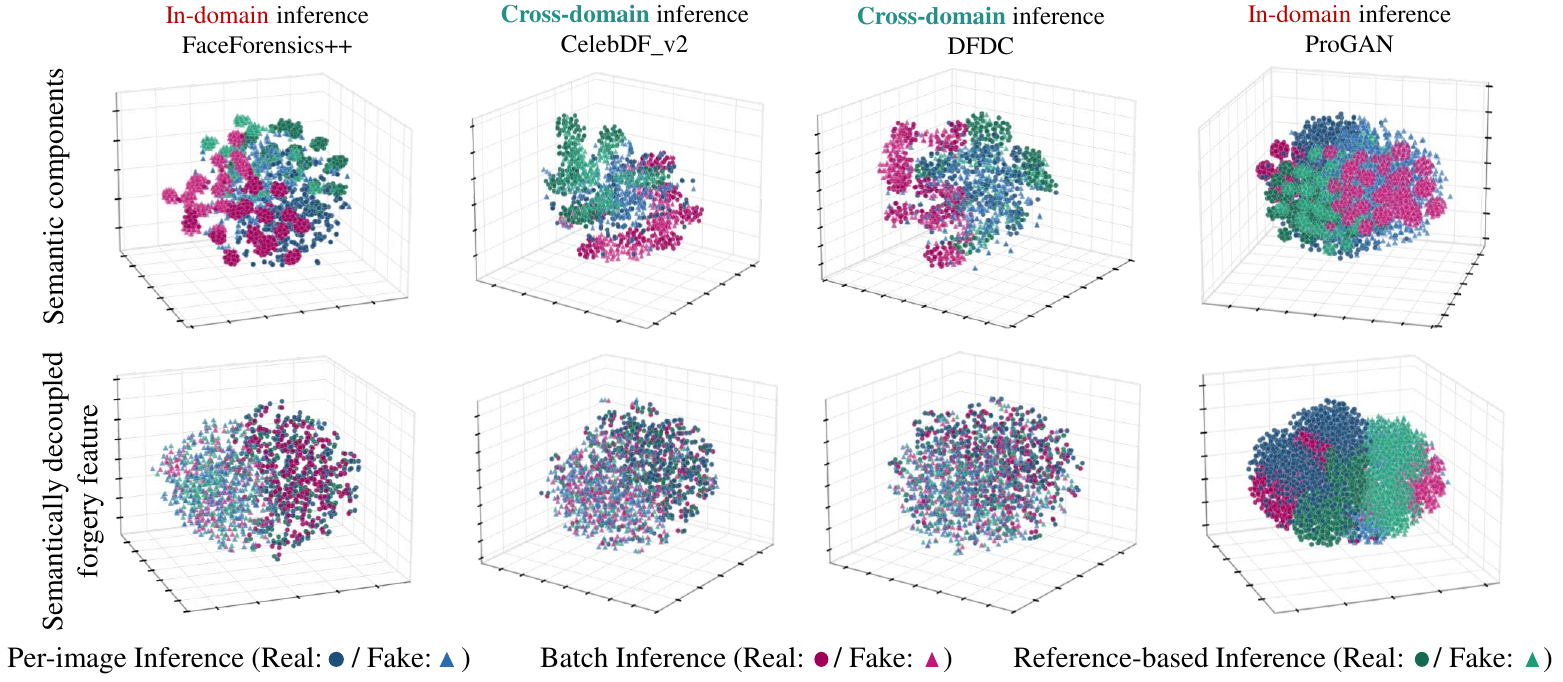}
    \caption{\textbf{Feature visualization of different GSD inference strategies.} We visualize the 3D t-SNE feature distributions under three deployment protocols, including Per-sample Inference, Batch Inference, and Reference-based Inference. The top row shows the semantic components preserved by the frozen semantic anchor, while the bottom row shows the semantically decoupled forgery features after GSD. Notably, across the three inference strategies, the semantic components exhibit highly similar feature spaces, while the semantically decoupled forgery features are mapped into an almost consistent forgery-oriented feature space.}
    \label{A1-3}
\end{figure*}

\section{More t-SNE Visualization}
\label{secA-3}
We provide additional t-SNE~\cite{tsne} visualizations of the naively fine-tuned CLIP and our GSD-CLIP to further examine whether semantic fallback is specific to face forgery or also appears in more general synthetic-image detection scenarios. As shown in Figure~\ref{A1-2}, we consider both face-forgery and natural-image synthesis settings. For face forgery, we include in-domain evaluation on FF++ and cross-domain evaluation on DFDC, FSGAN, and e4s. Beyond facial manipulation, we further analyze synthetic images from UniversalFakeDetect, including CycleGAN and StyleGAN subsets, where the semantic labels correspond to object or scene categories rather than face identities. Figures~\ref{A1-2}(a--c) show the feature distributions of the naively fine-tuned CLIP, while Figures~\ref{A1-2}(d--i) show the corresponding results after fine-tuning CLIP with our Geometric Semantic Decoupling module (\textbf{GSD}). 

From these visualizations, we make two key observations. First, the naively fine-tuned CLIP still exhibits a strong semantic organization in the feature space. Although the model has been fine-tuned for forgery detection, the right panels of Figures~\ref{A1-2}(a--c) show that samples with similar semantic labels tend to form compact local clusters. This phenomenon is visible not only in the face-forgery setting, where identity information remains a dominant factor, but also in the synthetic-image setting, where object-level categories continue to shape the representation geometry. Meanwhile, the real/fake separation in the left panels is not consistently aligned with these semantic clusters, especially under cross-domain evaluation. This indicates that downstream forensic fine-tuning does not fully reconstruct the feature space around manipulation-specific evidence. Instead, the high-level semantic structure inherited from CLIP pre-training remains a dominant geometric prior. We refer to this phenomenon as \textit{semantic fallback}: when facing forensic detection, the model tends to fall back to semantic concepts such as identity, object category, or scene type, rather than relying primarily on subtle manipulation artifacts.

Second, after introducing GSD, the feature space undergoes a clear reorganization. Compared with the naively fine-tuned CLIP, GSD-CLIP weakens compact semantic clustering and makes the representation less dominated by identity- or category-level semantics. This effect is particularly evident in the face-forgery datasets shown in Figures~\ref{A1-2}(d--g). After suppressing identity-dominant semantic directions, samples become less tightly grouped by face identity, and fake samples are more likely to deviate from their original identity-centered clusters. This suggests that authenticity-related cues, such as blending boundaries, texture discontinuities, and color inconsistencies~\cite{sbi, face-x-ray}, become more influential in shaping the learned representation.

For AI-generated natural images, the reorganization exhibits a slightly different pattern. As shown in Figures~\ref{A1-2}(h--i), both real and fake samples become less semantically compact after GSD, and the feature space becomes more dispersed rather than being reorganized into sharply separated semantic clusters. This is reasonable because synthetic-image artifacts are usually less spatially localized and less explicitly visible in the RGB domain than face-swapping artifacts. They may instead appear as subtle texture regularities, frequency-domain traces, or generator-specific distributional biases. Therefore, in natural-image synthesis detection, GSD mainly weakens the dominance of object- or scene-level clustering, allowing more diffuse authenticity-related variations to contribute to the feature geometry.

Overall, Figure~\ref{A1-2} shows that GSD reduces the dominance of inherited semantic structures and encourages the representation to rely less on identity, object category, or scene-level semantics. By explicitly decoupling semantic components from forensic features, GSD mitigates the semantic fallback of VFM-based detectors and yields a feature space that is more aligned with authenticity-related evidence, which is beneficial for cross-domain generalization.

\begin{figure*}[t]
    \centering
    \includegraphics[width=0.98\linewidth]{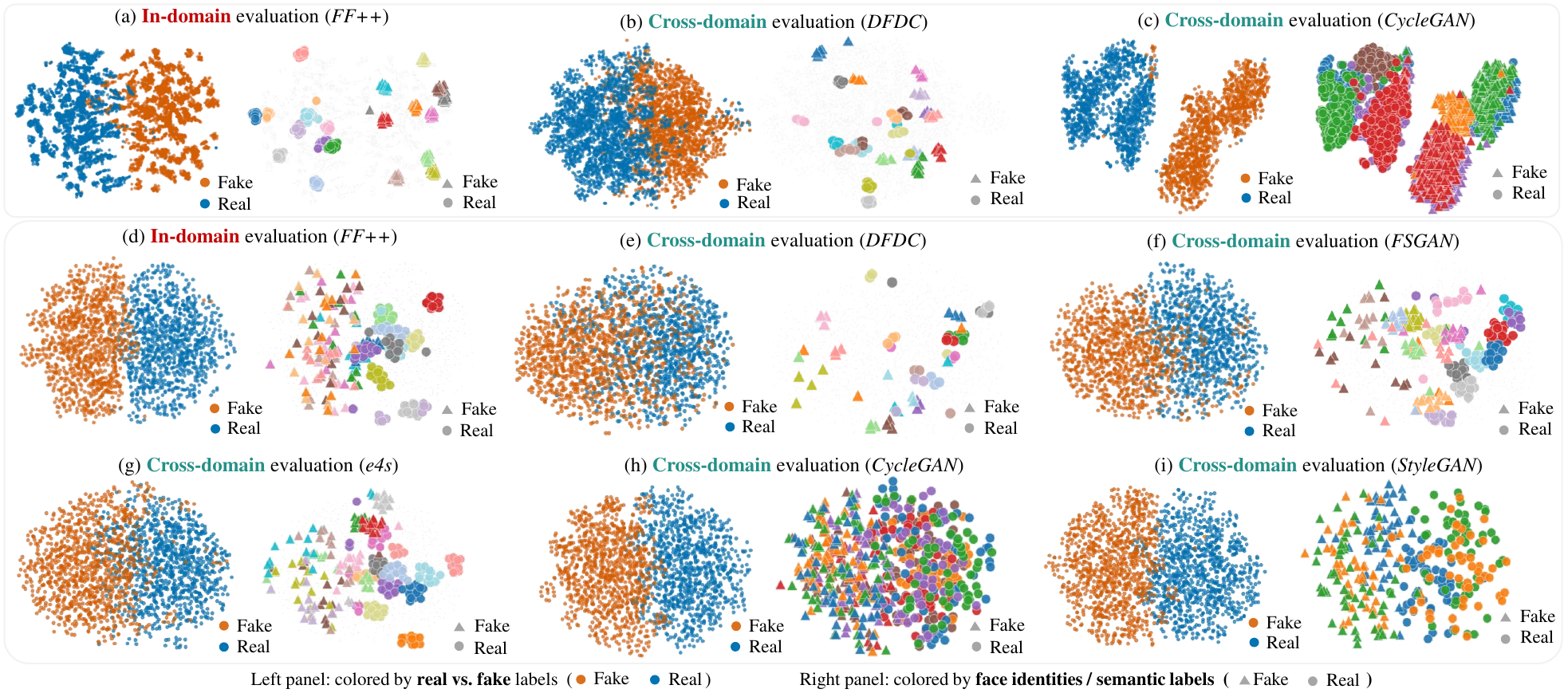}
    \caption{t-SNE~\cite{tsne} visualization of feature distributions extracted by the naively finetuned CLIP (a, b, c) and GSD-CLIP (d, e, f, g, h, i) on cross-domain datasets. Points are colored by forgery labels and semantic labels.}
    \label{A1-2}
\end{figure*}

\end{document}